\documentclass[10pt]{article}

\usepackage{times}
\usepackage{graphicx} % more modern
\graphicspath{{./pictures/}}
\DeclareGraphicsExtensions{.pdf,.jpeg,.png}

\usepackage{subfigure} 

\usepackage{algpseudocode}
\usepackage{algorithm}
\usepackage{algorithmicx}

\newcommand{\ran}{\ensuremath{\text{\rm Ran}}}
\newcommand{\nul}{\ensuremath{\text{\rm Ker}}}
\usepackage[table]{xcolor}
\usepackage{booktabs}
\usepackage{multirow}
\usepackage{multicol}

\usepackage[cmex10]{amsmath} % assumes amsmath package installed
\usepackage{amsthm}
\usepackage{amssymb}

\usepackage{url}

\newcommand\tstrut{\rule{0pt}{2.4ex}}
\newcommand\bstrut{\rule[-1.0ex]{0pt}{0pt}}

\newcommand{\argmin}{\operatornamewithlimits{argmin}}
\newtheorem{theorem}{Theorem}[section]
\newtheorem{lemma}[theorem]{Lemma}
\newtheorem{proposition}[theorem]{Proposition}
\newtheorem{corollary}[theorem]{Corollary}
\newtheorem{definition}[theorem]{Definition}

\newenvironment{remark}[1][]{\begin{trivlist}
\refstepcounter{theorem}
\item[\hskip \labelsep {\bfseries Remark \arabic{section}.\arabic{theorem}}\ifx&#1&.\else \ (#1).\fi]}{\end{trivlist}}

\newcolumntype{a}{>{\columncolor{gray!35}}c}

\title{Convex Learning of Multiple Tasks and their Structure}
\author{Carlo Ciliberto  \thanks{
       Laboratory for Computational and Statistical Learning, Istituto Italiano di Tecnologia,
       Via Morego, 30,
       16100, Genova, Italy, ({\tt cciliber@mit.edu})} \ $^\dag$   
 \and
Youssef Mroueh  $^*$\thanks{
       Poggio Lab, Massachusetts Institute of Technology,
       Massachusetts Ave. 77,
       02139, Cambridge, MA, USA ({\tt tp@ai.mit.edu})}
 \and
Tomaso Poggio  $^*$ $^\dag$
 \and
Lorenzo Rosasco  $^*$ $^\dag$\thanks{
       DIBRIS, Universit\`a di Genova,
       Via Dodecaneso, 35,
       16146, Genova, Italy, ({\tt lrosasco@mit.edu})}    
}
\date{}

\begin{document} 
\maketitle

\begin{abstract}
  Reducing the amount of human supervision is a key problem in machine learning and a natural  approach is that of exploiting  the relations  (structure) among  different tasks. This is the idea at the core of multi-task learning. In this context a fundamental question is how to incorporate the tasks structure in the learning problem.
  We tackle this question by studying a general computational framework that allows to encode a-priori knowledge of the tasks structure in the form of a convex penalty; in this setting a variety of previously proposed methods can be recovered as special cases, including linear and non-linear approaches. Within this framework, we show that tasks and their structure can be efficiently learned considering a convex optimization problem that can be approached by means of block coordinate methods such as alternating minimization and for which we prove convergence to the global minimum.

%  Our approach allows to consider 
%  
%  
%   {\color{blue}The framework we consider and the meta-optimization method we propose can be naturally expressed in both a dual and primal formulation. This allows for the straightforward extension of previous multi-task methods to the non-linear settings and furthermore provides with practical optimization algorithms to solve them.}
  %Another focus of the paper is on methods for learning interpretable tasks structure. To this end we propose a novel setting whose aim is to recover sparse sets of relations among tasks.
\end{abstract}

\section{Introduction}

Current  machine learning systems  achieve remarkable results in several challenging  tasks, but are limited by the amount of human supervision required. Leveraging similarity among different problems is widely acknowledged to be  a key approach to reduce the need for supervised data. Indeed, this idea is at the basis of multi-task learning, where the joint solution of different  problems (tasks) has the potential to exploit tasks relatedness (structure) to improve learning accuracy. This idea has motivated a variety of methods, including frequentist \cite{micchelli04,argyriou08,argyriou08b} and Bayesian methods (see e.g. \cite{alvarez12} and references therein), with connections to structured learning~\cite{bakir07,tsochantaridis04}. \\
The focus of our study is the development of a general regularization  framework  to   learn multiple tasks as well as their structure.  Following~\cite{micchelli04,evgeniou05} we consider a setting where tasks are modeled as the components of a vector-valued function and their structure corresponds to the choice of suitable functional spaces. Exploiting the theory of reproducing kernel Hilbert spaces for vector-valued functions (RKHSvv)~\cite{micchelli04},  we consider and analyze  a flexible regularization framework, within which a variety of previously proposed approaches can be recovered as special cases, see e.g.  ~\cite{jacob08,lozano10,minh11,zhang10,dinuzzo11,sindhwani13}. Our main technical contribution  is a unifying  study of the minimization problem corresponding to such a  regularization framework. More precisely, we devise an optimization approach that can efficiently compute a solution and for which we prove convergence under weak assumptions. Our approach is based on a barrier method that is combined with block coordinate descent techniques \cite{tseng01,razaviyayn13}. In this sense our analysis generalizes the results in \cite{argyriou08} for which a low-rank assumption was considered; however the extension is not straightforward, since we consider a much larger class of regularization schemes (any convex penalty). Up to our knowledge, this is the first result in multi-task learning proving the convergence of alternating minimization schemes for such a general family of problems.\\
The RKHSvv setting allows to naturally deal both with linear and non-linear models and the approach we propose provides a general computational framework for learning output kernels as formalized in \cite{dinuzzo11}.\\
The rest of the paper is organized as follows: in Sec~\ref{sec:RKHSvv} we review basic ideas of regularization in RKHSvv. In Sec.~\ref{sec:unified_perspective} we discuss the equivalence of different approaches to encode known structures among multiple tasks. In Sec.~\ref{sec:learn_joint} we discuss a general framework for learning multiple tasks and their relations where we consider a wide family of structure-inducing penalties and study an optimization strategy to solve them. This setting allows us, in Sec.~\ref{sec:examples}, to recover several previous methods as special cases. Finally in Sec.~\ref{sec:experiments} we evaluate the performance of the optimization method proposed.

\paragraph{Notation.} With $S^n_{++} \subset S^n_+ \subset S^n \subset \mathbb{R}^{n \times n}$ we denote respectively the space of positive definite, positive semidefinite (PSD) and symmetric $n \times n$ real-valued matrices. $O^n$ denotes the space of orthonormal $n \times n$ matrices. For any square matrix $M\in\mathbb{R}^{n \times n}$ and $p\geq1$, we denote by $\|M\|_p = (\sum_{i=1}^n \sigma_i(M)^p)^{1/p}$ the $p$-Schatten norm of $M$, where $\sigma_i(M)$ is the $i$-th largest singular value of $M$. For any $M\in\mathbb{R}^{n \times m}$, $M^\top$ denotes the transpose of $M$. For any PSD matrix $A\in S_+^n$, $A^\dagger$ denotes the pseudoinverse of $A$.  We denote by $I_n\in S_{++}^n$ the $n \times n$ identity matrix. The notation $\ran(M)\subseteq\mathbb{R}^m$ identifies the range of columns of a matrix $M\in\mathbb{R}^{m \times n}$.

\section{Background}\label{sec:RKHSvv}

We study the problem of jointly learning multiple tasks by modeling individual task-predictors as the components of a vector-valued function. Let us assume to have $T$ supervised scalar learning problems (or tasks), each with a ``training'' set of input-output observations $\mathcal{S}_t = \{(x_{it},y_{it})\}_{i=1}^{n_t}$ with $x_{it} \in \mathcal{X}$ input space and $y_{it}\in\mathcal{Y}$ output space\footnote{To avoid clutter in the notation, we have restricted ourselves to the typical situation where all tasks share same input and output spaces, i.e. $\mathcal{X}_t = \mathcal{X}$ and $\mathcal{Y}_t \subseteq \mathbb{R}$.}. Given a loss function $\mathcal{L}:\mathbb{R}\times\mathbb{R}\to\mathbb{R}_+$ that measures the per-task prediction errors, we want to solve the following joint regularized learning problem
\begin{equation}\label{eq:learning_problem}
\underset{f\in\mathcal{H}}{\text{minimize}} \ \ \sum_{t=1}^T \frac{1}{n_t}\sum_{i=1}^{n_t} \mathcal{L}(y_{i}^{(t)},f_t(x_{i}^{(t)})) + \lambda \|f\|_\mathcal{H}^2
\end{equation}
where $\mathcal{H}$ is an Hilbert space of vector-valued functions $f:\mathcal{X} \to \mathcal{Y}^T$with scalar components $f_t:\mathcal{X}\to\mathcal{Y}$. In order to define a suitable space of hypotheses $\mathcal{H}$, in this section we briefly recall concepts from the theory of reproducing kernel Hilbert spaces for vector-valued functions (RKHSvv) and corresponding regularization theory, which plays a key role in our work. In particular, we focus on a class of reproducing kernels (known as separable kernels) that can be designed to encode specific tasks structures (see ~\cite{evgeniou05,argyriou13} and Sec.~\ref{sec:unified_perspective}). Interestingly, separable kernels are related to ideas such as defining a metric on the output space or a label encoding in multi-label problems (see Sec.~\ref{sec:unified_perspective})

\begin{remark}[Multi-task and multi-label learning]
Multi-label learning is a class of supervised learning problems in which the goal is to associate input examples with a label or a set of labels chosen from a discrete set. In general, due to discrete nature of the output space, these problems cannot be solved directly; hence, a so-called {\it surrogate} problem is often introduced, which is computationally tractable and whose solution allows to recover the solution of the original problem~\cite{steinwart08,bartlett06,mroueh12}.\\
Multi-label learning and multi-task learning are strongly related. Indeed, surrogate problems typically consist in a set of distinct supervised learning problems (or tasks) that are solved simultaneously and therefore have a natural formulation in the multi-task setting. For instance, in multi-class classification problems the ``One vs All'' strategy is often adopted, which consists in solving a set of multiple binary classification problems, one for each class.
\end{remark}

\subsection{Learning Multiple Tasks with RKHSvv}

In the scalar setting, reproducing kernel Hilbert spaces have already been proved to be a powerful tool for machine learning applications. Interestingly, the theory of RKHSvv and corresponding Tikhonov regularization scheme follow closely the derivation in the scalar case. 

\begin{definition}
Let $(\mathcal{H},\langle\cdot,\cdot\rangle_{\mathcal{H}})$ be a Hilbert space of functions from $\mathcal{X}$ to $\mathbb{R}^T$. A symmetric, positive definite, matrix-valued function $\Gamma: \mathcal{X} \times \mathcal{X} \to \mathbb{R}^{T \times T}$ is called a reproducing kernel for $\mathcal{H}$ if for all $x \in \mathcal{X}, c \in \mathbb{R}^T$ and $f \in \mathcal{H}$ we have that $\Gamma(x,\cdot)c \in \mathcal{H}$ and the following reproducing property holds
$
\langle f(x), c\rangle_{\mathbb{R}^T} = \langle f,\Gamma(x,\cdot)c\rangle_{\mathcal{H}}.
$
\end{definition}
In analogy to the scalar setting, it can be proved (see~\cite{micchelli04}) that the Representer Theorem holds also for regularization in RKHSvv. In particular we have that any solution of the learning problem introduced in Eq.~\eqref{eq:learning_problem} can be written in the form 
\begin{equation}
f(x) = \sum_{t = 1}^T \sum_{i=1}^{n_t} \Gamma(x,x_{i}^{(t)}) c_{i}^{(t)}   
\end{equation}
with $c_i^{(t)} \in \mathbb{R}^T$ coefficient vectors.\\
The choice of kernel $\Gamma$ induces a joint representation of the inputs as well as a structure among the output components~\cite{alvarez12}; In the rest of the paper we will focus on so-called separable kernels, where these two aspects are factorized. In Section~\ref{sec:learn_joint}, we will see how separable kernels provide a natural way to learn the tasks structure as well as the tasks. 

 % \textbf{Remark.}  Vector-valued functions can be modeled as well in the framework of Gaussian processes.  For a detailed discussion on the connections between RKHSvv and Gaussian processes, we refer the reader to~\cite{alvarez12}.

%belongs to the span of functions $\{\Gamma(x_i,\cdot)e_j\}_{j,i=1}^{T,n}\subset\mathcal{H}$ with $e_j\in\mathbb{R}^T$ the $j$-th vector of the canonical basis of $\mathbb{R}^T$. 

\subsection{Separable Kernels}
Separable (reproducing) kernels are functions of the form  
$\Gamma(x,x')=k(x,x')A$ \ $\forall x,x'\in\mathcal{X}$ where $k:\mathcal{X}\times\mathcal{X}\to\mathbb{R}$ is a scalar reproducing kernel and $A\in S_+^T$ is a positive semi-definite (PSD) matrix. In this case, the representer theorem allows to rewrite problem~\eqref{eq:learning_problem} in a more compact matrix notation as
\begin{equation}\label{eq:learning_problem_matrix}\tag{$\mathcal{P}$}
\underset{C\in\mathbb{R}^{n \times T}}{\text{minimize}} \ \ V(Y,KCA) + \lambda \ tr(AC^\top K C).
\end{equation}
Here $Y\in\mathbb{R}^{n \times T}$ is a matrix with $n = \sum_{t=1}^T n_t$ rows containing the output  points; $K\in S_+^n$ is the empirical kernel matrix associated to $k$ and $V:\mathbb{R}^{n \times T} \times \mathbb{R}^{n \times T} \to \mathbb{R}_+$ generalizes the loss in~\eqref{eq:learning_problem} and consists in a linear combination of the entry-wise application of $\mathcal{L}$. Notice that this formulation accounts also the situation where not all training outputs $y^{(t)}$ are observed when a given input $x \in \mathcal{X}$ is provided: in this case the functional $V$ weights $0$ the loss values of those entries of $Y$ (and the associated entries of $KCA$) that are not available in training.\\
Finally, the second term in~\eqref{eq:learning_problem_matrix} follows by observing that,  for all  $f\in\mathcal{H}$ of the form $f(\cdot) =\sum_{i=1}^n k(x_i,\cdot)A c_i$, the squared norm  can be written as $\|f\|_\mathcal{H}^2=\sum_{i,j}^n k(x_i,x_j) c_i^\top A c_j = tr(AC^\top KC)$ where $C\in\mathbb{R}^{n \times T}$ is the matrix with $i$-th row corresponding to the coefficient vector $c_i\in\mathbb{R}^T$ of $f$. Notice that we have re-ordered the index $i$ to be in $\{1,\dots,n\}$ to ease the notation.

\subsection{Incorporating Known Tasks Structure}\label{sec:unified_perspective}
Separable kernels provide a natural way to incorporate the task structure when the latter is known a priori. This strategy is quite general and indeed in the following we comment on how the matrix $A$ can be chosen to recover several multi-task methods previously proposed in contexts such as regularization, coding/embeddings or output metric learning, postponing a more detailed discussion in the supplementary material. These observations motivate the extension in Sec.~\ref{sec:learn_joint} of the learning problem~\eqref{eq:learning_problem_matrix} to a setting where it is possible to infer $A$ from the data.

\paragraph{Regularizers.}
Tasks relations can be enforced by devising suitable regularizers~\cite{evgeniou05}. Interestingly, for a large class of such methods it can be shown that this is equivalent to the choice of the matrix $A$ (or rather its pseudoinverse) \cite{micchelli04}. If we consider the squared norm of a function $f = \sum_{i=1}^n k(x_i,\cdot)Ac_i \in\mathcal{H}$ we have (see~\cite{evgeniou05})
\begin{equation}
\|f\|_\mathcal{H}^2 = \sum_{t,s=1}^T A^\dagger_{ts} \langle f_t, f_s \rangle_{\mathcal{H}_k}
\end{equation}
where $A_t$ is the $t$-th column of $A$, $\mathcal{H}_k$ is the RKHS associated to the scalar kernel $k$ and $f_t= \sum_{i=1}^n k(x_i,\cdot) A_t^\top c_i \in\mathcal{H}_k$  is the $t$-th component of $f$. The above equation suggests to interpret $A^\dagger$ as the matrix that models the structural relations between tasks by directly coupling different predictors. 
For instance, by setting $A^\dagger= I_T + \gamma (\mathbf{1}\mathbf{1}^\top)/T$, with $\mathbf{1}\in\mathbb{R}^T$ the vector of all $1$s, we have that the parameter $\gamma$ controls the variance $\sum_{t=1}^T \|\bar{f} - f_t\|_{\mathcal{H}_k}^2$ of the tasks with respect to their mean $\bar{f}=\frac{1}{T} \sum_{t=1}^T f_t$. If we have access to some notion of similarity among tasks in the form of a graph with adjacency matrix $W\in S^T$, we can consider the regularizer $\sum_{t,s=1}^T W_{t,s} \|f_t - f_s\|_{\mathcal{H}_k}^2 + \gamma \sum_{t}^T \|f_t\|_{\mathcal{H}_k}^2$ which corresponds to $A^\dagger=L + \gamma I_T$ with $L$ the graph Laplacian induced by $W$.
%For further examples and discussion we refer to~\cite{micchelli04}.
% By means of simple algebra it can be shown that $\|f\|_\mathcal{H}^2=\sum_{t,s=1}^T A^\dagger \langle f_t,f_s \rangle_\mathcal{H}$ with $f_t$ the $t$-th task in $f$. Therefore $A^\dagger$ can be interpreted as encoding tasks relations by directly coupling individual predictors and consequently it can be modeled according to the available a-priori information. For instance, if we have access to tasks similarities in the form of a graph with adjacency matrix $W\in S^T$, we can impose the natural regularizer $\sum_{t,s=1}^T W_{t,s} \|f_t - f_s\|_{\mathcal{H}_k}^2 + \gamma \sum_{t}^T \|f_t\|_{\mathcal{H}_k}^2$ by choosing $A=(L + \gamma I_T)^{-1}$ where $L$ is the graph laplacian induced by $W$. For further discussion and examples see~\cite{micchelli04}.
\paragraph{Output Metric.}
A different approach to model tasks relatedness consists in choosing a suitable  metric on the output space to reflect the tasks structure~\cite{lozano10}.  Clearly a change of metric on the output space with the standard inner product $\langle y,y\prime \rangle_{\mathbb{R}^T}$ between two output points $y,y\prime\in\mathcal{Y}^T$ corresponds to the choice of a different inner product $\langle y,y\prime\rangle_\Theta = \langle y, \theta y\prime \rangle_{\mathbb{R}^T}$ for some positive definite matrix $\Theta \in S_{++}^T$. Indeed this can be direct related to the choice of a suitable separable kernel.  In particular, for the least squares loss function a direct equivalence holds between choosing a metric deformation associated to a $\Theta\in S_{++}^T$ and a separable kernel $k(\cdot,\cdot)I_T$ or use the canonical metric (i.e. with $\Theta=I_T$ the identity) and kernel $k(\cdot,\cdot)\Theta$. The details of this equivalence can be found in the supplementary material.
\paragraph{Output Representation.}
The tasks structure can also be modeled by designing an ad-hoc embedding for the output space. This approach is particularly useful for multi-label scenarios, where output embedding can be designed to encode complex structures such as (e.g. trees, strings, graphs, etc.)~\cite{fergus10,joachims09,crammer00}. Interestingly in these cases, or more generally whenever the embedding map $L:\mathcal{Y}^T\to\widetilde{\mathcal{Y}}$, from the original to the new output space, is linear, then it is possible to show that the learning problem with new code is equivalent to~\eqref{eq:learning_problem} for a suitable choice of separable kernel with $A=L^\top L$. We refer again to the supplementary material for the details of this equivalence.

%\paragraph{Gaussian Processes.} Connections between Gaussian Processes and RKHS are well-known and extend straightforwardly to the vector-valued setting (even for non separable kernels/covariance functions). For a detailed discussion of the connections between the two frameworks we refer the reader to~\cite{alvarez12}
%~\cite{ecoc,trees?} where the authors designed novel mappings from the original output space $\mathcal{Y}$ to a new $\widetilde{\mathcal{Y}}\subseteq\mathbb{R}^S$. This approach can be particularly appealing in multi-class scenarios, since class labels are often arbitrarily chosen and careful design of the output the codes has been proved to improve classification performance (e.g.~\cite{code-trees?,simplex-coding}).

\section{Learning the Tasks and their Structure}\label{sec:learn_joint}
Clearly, an interesting setting occurs when knowledge of the tasks structure is not available and therefore it is not possible to design a suitable separable kernel. In this case a favorable approach is to infer the tasks relations directly from the data. To this end we propose to consider the following extension of problem~\eqref{eq:learning_problem_matrix}
\begin{equation}\label{eq:nonconvex}\tag{$\mathcal{Q}$}
\begin{aligned}
\underset{C\in\mathbb{R}^{n \times T}, A \in S_+^T}{\text{minimize}} & \ \  
V(Y,KCA) + \lambda tr(AC^\top KC) + F(A),
% \underset{\substack{C\in\mathbb{R}^{n \times T},\\ A \in S_+^T}}{\text{min.}} & 
% V(Y,KCA) + \lambda tr(AC^\top KC) + \gamma \ F(A),
\end{aligned}
\end{equation}
where the penalty $F:S_+^T\to\mathbb{R}_+$ is designed to learn specific tasks structures encoded in the matrix $A$. The above regularization is general enough to encompass a large number of previously proposed approaches by simply specifying a choice of the scalar kernel and the penalty $F$. A detailed discussion of these connections is postponed to Section~\ref{sec:examples}.
%A variety of multi-task methods have proposed different choices of such function in order to recover special tasks-structure. Some of these approaches as well as a new regularizer are discussed in Section~\ref{sec:examples}. 
In this section, we focus on computational aspects. Throughout, we restrict ourselves to convex loss functions $V$ and convex (and coercive) penalties $F$.  In this case, the objective function in~\eqref{eq:nonconvex} is separately convex  in  $C$ and $A$ but not  jointly convex. Hence,   block coordinate methods, which are often used in practice, e.g. alternating minimization over $C$ and $A$,  are not guaranteed to converge to a global minimum. Our  study provides a general  framework to provably compute a solution to problem~\eqref{eq:nonconvex}. First, In Section \ref{CMBM}, we prove our main results providing a  characterization  of the solutions of Problem~\eqref{eq:nonconvex} and  studying a barrier method to cast their computation as a convex optimization problem. Second, in Section \ref{BCM}, we discuss how block coordinate methods can be naturally used to solve such a problem, analyze their convergence properties and discuss some general cases of interest.

\subsection{Characterization of Minima and A Barrier Method}\label{CMBM}

We begin, in Section \ref{QR}, providing a characterization of the solutions to Problem \eqref{eq:nonconvex} by showing that it has an equivalent formulation in terms of the minimization of a convex objective function, namely Problem \eqref{eq:convex_equivalence}.  
Depending on the behavior of the objective function on  the boundary of the optimization domain,  
Problem \eqref{eq:convex_equivalence} might not be solved using standard optimization techniques. This possible issue motivates the introduction, in Section \ref{RSd}, of a barrier method; a family of ``perturbated'' convex programs is introduced whose solutions are shown to converge to those of Problem \eqref{eq:convex_equivalence} (and hence of the original \eqref{eq:nonconvex}).

\subsubsection{An Equivalent formulation for~\eqref{eq:nonconvex}}\label{QR}

The objective functional in~\eqref{eq:nonconvex} is not convex, therefore in principle it is hard to find a global minimizer. As it turns out however, it is possible to circumvent this issue and efficiently find a global solution to~\eqref{eq:nonconvex}. The following result represents a first step in this direction. 
%
%However, it is possible to show that there exists a convex functional whose minimization leads to a problem equivalent to~\eqref{eq:nonconvex}. This result, which is one of the main contributions of this work, allows to focus our attention towards a problem with convex functional, which is typically much easier to solve.

\begin{theorem}\label{teo:convex_equivalence}
Let $K\in S_+^n$ and consider the convex set
$$
\mathcal{C}=\left\{(C,A) \in \mathbb{R}^{n \times T} \times S^T_+ \ | \ \ran(C^\top KC) \subseteq \ran(A) \right\}.
$$
Then, for any $F:S_+^T\to \mathbb{R}_+$ convex and coercive, problem
\begin{equation}\label{eq:convex_equivalence}\tag{$\mathcal{R}$}
\begin{aligned}
\underset{(C,A) \ \in \ \mathcal{C}}{\mbox{\emph{minimize}}} &V(Y,KC)+\lambda tr\left(A^{\dagger}C^\top K C\right) + F(A)
\end{aligned}
\end{equation}
has convex objective function and it is equivalent to \eqref{eq:nonconvex}. In particular, the two problems achieve the same minimum value and, given a solution $(C_R,A_R)$ for \eqref{eq:convex_equivalence}, the couple $(C_RA_R^\dagger,A_R)$ is a minimizer for \eqref{eq:nonconvex}. Vice-versa, given a solution $(C_Q,A_Q)$ for \eqref{eq:nonconvex}, the couple $(C_QA_Q,A_Q)$ is a minimizer for \eqref{eq:convex_equivalence}. 
\end{theorem}
The above result highlights a remarkable connection between the problems \eqref{eq:nonconvex} (non-convex) and \eqref{eq:convex_equivalence} (convex). In particular, we have the following Corollary, which provides us with a useful characterization of the local minimizers of problem~\eqref{eq:nonconvex}. 

\begin{corollary}\label{cor:invexity} 
Let $Q:\mathbb{R}^{n \times T} \times S_+^T \to \mathbb{R}$ be the objective  function of problem~\eqref{eq:nonconvex}. Then, every local minimizer for $Q$ on the open set $\mathbb{R}^{n \times T} \times S_{++}^T$ is also a global minimizer. 
\end{corollary}
Corollary~\ref{cor:invexity} follows from Theorem~\ref{teo:convex_equivalence} and the fact that, on the restricted domain $\mathbb{R}^{n \times T} \times S_{++}^T$, the map $Q$ is the combination of the objective functional of~\eqref{eq:convex_equivalence} and the invertible function $(C,A)\longmapsto(CA,A)$.  Moreover, if $Q$ is differentiable,  i.e. $V$ and the penalty $F$ are differentiable, this is exactly the definition of a \textit{convexifiable} function, which in particular implies  {\em invexity}~\cite{craven95}. 
The latter property ensures that, in the differentiable case, all the {\em stationary} points (rather than only local minimizers) are global minimizers. This result was originally proved in~\cite{dinuzzo11} for the special case of $V$ the least-squares loss and $F(\cdot)=\|\cdot\|_F^2$ the Frobenius norm; Here we have proved its generalization to all convex losses $V$ and penalties $F$.\\
We end this section adding two comments. First, we note that, while the objective function in Problem \eqref{eq:convex_equivalence} is convex, the corresponding minimization problem might not be a convex program  (in the sense that  the feasible set~$\mathcal{C}$ is not identified by a set of linear equalities and non-linear convex inequalities~\cite{boyd04}). Second, Corollary~\eqref{cor:invexity} holds only on the interior of the minimization domain $\mathbb{R}^{n \times T} \times S_+^T$ and does not characterize the behavior of the target functional on its boundary. In fact, one can see that both issues can be tackled defining a {\em perturbed} objective functional having a suitable behavior on the 
boundary of the minimization domain. This is the key motivation for the barrier method we discuss in the next section.

\subsubsection{A Barrier Method to Optimize~\eqref{eq:convex_equivalence}}\label{RSd}
 
Here we propose a barrier approach inspired by the work in~\cite{argyriou08} by introducing a perturbation of problem~\eqref{eq:convex_equivalence} that enforces the objective functions to be equal to $+\infty$ on the boundary of $\mathbb{R}^{n \times T} \times S_+^T$. As a consequence, each perturbed problem can be solved as a convex optimization constrained on a closed cone. %A pictorial representation summarizing our analysis from \eqref{eq:nonconvex} to \eqref{eq:convex_equivalence} and finally to \eqref{eq:perturbation} is reported in Figure~\ref{fig:pictorial}. 
The latter comment is made more precise in the following result that we prove in the supplementary material.

\begin{theorem}\label{teo:perturbation}
Consider the family of optimization problems
\begin{equation}\label{eq:perturbation}\tag{$\mathcal{S}^\delta$}
\begin{aligned}
\underset{\substack{C\in\mathbb{R}^{n \times T}, \\ A\in S_{+}^T}}{\mbox{\emph{minimize}}} & V(Y,KC) +\lambda tr(A^{-1}(C^\top KC + \delta^2 I_T) ) + F(A)
\end{aligned}
\end{equation}
with $I_T \in S_{++}^T$ the identity matrix. Then, for each $\delta>0$ the problem~\eqref{eq:perturbation} admits a minimum. Furthermore, the set of minimizers for~\eqref{eq:perturbation} converges to the set of minimizers for~\eqref{eq:convex_equivalence} as $\delta$ tends to zero. More precisely, given any sequence $\delta_m>0$ such that $\delta_m\to0$ and a sequence of minimizers $(C_m,A_m)\in\mathbb{R}^{n \times T} \times S_{+}^T$ for~\eqref{eq:perturbation}, there exists a sequence $(C^*_m,A^*_m)\in\mathbb{R}^{n \times T} \times S_{+}^T$ of minimizers for~\eqref{eq:convex_equivalence} such that $\|C_m-C^*_m\|_F + \|A_m-A^*_m\|_F \to0$ as $m\to+\infty$.
\end{theorem}
The barrier $\delta^2 tr(A^{-1})$ is fairly natural and can be seen as preconditioning of the problem leading to favorable computations. The proposed barrier method is similar in spirit to the approach developed in~\cite{argyriou08} and indeed Theorem~\ref{teo:perturbation} and next Corollary~\ref{cor:bcd} are a generalization over the two main results in~\cite{argyriou08} to any convex penalty $F$ on the cone of PSD matrices. However, notice that since we are considering a much wider family of penalties (than the trace norm as in~\cite{argyriou08}) our results cannot directly derived from those in~\cite{argyriou08}. In the next section we discuss how to compute the solution of Problem \eqref{eq:perturbation} considering a block coordinate approach.
% Indeed in the next section we discuss how to compute the solution of Problem \eqref{eq:nonconvex} considering a block coordinate approach.

\subsection{Block Coordinate  Descent Methods}\label{BCM}

The characteristic block variable structure of the objective function in problem \eqref{eq:perturbation}, suggests that it might be beneficial to use  block coordinate methods (BCM) (see~\cite{beck11}) to solve it. Here with  BCM we identify a large class of  methods that, in our setting, iterate steps of an optimization on $C$, with $A$ fixed, followed by an optimization of $A$, for $C$ fixed.\\
A {\em meta} block coordinate algorithm to solve~\eqref{eq:perturbation} is reported in in Algorithm~\ref{alg:bcd}. Here we interpret each optimization step over $C$ as a supervised step, and each optimization step over $A$ as a  an unsupervised step  (in the sense that it involves the inputs but not the outputs). Indeed, when the structure matrix $A$ is fixed, problem~\eqref{eq:convex_equivalence} boils down to the standard supervised multi-task learning frameworks where a priori knowledge regarding the tasks structure is available. Instead, when the coefficient matrix $C$ is fixed, the problem of learning $A$ can be interpreted as an unsupervised setting in which the goal is to actually find the underlying task structure \cite{tenenbaum10}.\\
Several optimization methods can be used as procedures for both \textsc{SupervisedStep} and \textsc{UnsupervisedStep} in Algorithm~\ref{alg:bcd}.  In particular, a first class of methods is called Block Coordinate Descent (BCD) and identifies a wide class of iterative methods that perform (typically inexact) minimization of the objective function one block of variables at the time. Different strategies to choose which direction minimize at each step have been proposed: pre-fixed cyclic order, greedy search~\cite{razaviyayn13} or randomly, according to a predetermined distribution~\cite{nesterov12}. For a review of several BCD algorithms we refer the reader to~\cite{razaviyayn13} and references therein.\\
A second class of methods is called alternating minimization and corresponds to the situation where at each step in Algorithm~\ref{alg:bcd} and exact minimization is performed.  This latter approach is favorable when a closed form solution exists for at least one block of variables (see Section \ref{CF}) and has been  studied  extensively in \cite{tseng01} in the abstract setting where an oracle provides a block-wise minimizer at each iteration.
% For simplicity, in this work we will restrict to the case in which the two procedures both consists on a single step of (Projected) Gradient Descent~\cite{boyd04}, or block-wise proximal methods~\cite{razaviyayn13}. 
% Notice however that a special case of BCD, which falls also under the name of alternate minimization, consists in completing the descent along one block until the block-wise minimum is achieved. Alternate minimization has been studied in~\cite{tseng01} for the more general situation where an oracle provides a block-wise minimizer at each iteration. 
The following Corollary describes  the convergence properties of BCD and Alternate minimization sequences provided by applying Algorithm~\ref{alg:bcd} to~\eqref{eq:perturbation}.

\begin{algorithm}[t]
   \caption{\textsc{Convex Multi-task Learning}}
   \label{alg:bcd}
\begin{algorithmic}
   \State {\bfseries Input:} $K, Y,\epsilon$ tolerance, $\delta$ perturbation parameter, $S$ objective functional of~\eqref{eq:perturbation}, $V$ loss, $F$ structure penalty.
   \State {\bfseries Initialize:} $(C,A)=(C_0,A_0), t=0$ 
   \Repeat
   %\STATE Initialize $noChange = true$.
   %\FOR{$i=1$ {\bfseries to} $m-1$}
  % \IF{$x_i > x_{i+1}$} 
   \State $C_{t+1} \gets$ \textsc{SupervisedStep} $(V,K,Y,C_{t},A_{t})$
    \State $A_{t+1} \gets$ \textsc{UnsupervisedStep}$(F,K,\delta,C_{t+1},A_{t})$
     \State $t \gets t+1$
   %\ENDIF
   %\ENDFOR
   %\Until{$\|(C_{t},A_{t})-(C_{t+1},A_{t+1})\|<\epsilon$}
   \Until{$| S(C_{t+1},A_{t+1})-S(C_{t},A_{t}) | < \epsilon$}
\end{algorithmic}
\end{algorithm}

\begin{corollary}\label{cor:bcd}
Let the Problem~\eqref{eq:perturbation} be defined as in Theorem~\ref{teo:perturbation} then:
\begin{itemize}
\item[(a)] \textbf{Alternating Minimization:} Let the two procedures in Algorithm~\ref{alg:bcd} each provide a block-wise minimizer of the functional with the other block held fixed.  Then every limiting point of a minimization sequence provided by Algorithm~\ref{alg:bcd}, is a global minimizer for~\eqref{eq:perturbation}.
\item[(b)] \textbf{Block Coordinate Descent:} Let the two procedures in Algorithm~\ref{alg:bcd} each consist in a single step of a first order optimization method (e.g. Projected Gradient Descent, Proximal methods, etc.). Then every limiting point of a minimizing sequence provided by Algorithm~\ref{alg:bcd} is a global minimizer for~\eqref{eq:perturbation}.
\end{itemize}
\end{corollary}

Corollary~\eqref{cor:bcd} follows by applying previous results on BCD and Alternate minimization. In particular, for the proof of part $(a)$ we refer to Theorem $4.1$ in ~\cite{tseng01}, while for part $(b)$ we refer to Theorem $2$ in~\cite{razaviyayn13}.\\
In the following we discuss the actual implementation of both \textsc{Supervised} and \textsc{Unsupervised} procedures in the case where $V$ is chosen to be least-squares loss and the penalty $F$ to be a spectral $p$-Schatten norm. This should provide the reader with a practical example of how the meta-algorithm introduced in this section can be specialized to a specific multi-task learning setting.

\begin{remark}(Convergence of Block Coordinate Methods)
Several works in multi-task learning have proposed some form of BCM strategy to solve the learning problem. However, up to our knowledge, so far only the authors in~\cite{argyriou08} have considered the issue of convergence to a global optimum. Their results where proved for a specific choice of structure penalty in a framework similar to that of problem~\eqref{eq:convex_equivalence} (see Section \ref{sec:examples}) but do not extend straightforwardly to other settings. Corollary~\ref{cor:bcd} aims to fill this gap, providing convergence guarantees for block coordinate methods for a large class of multi-task learning problems.
\end{remark}

\subsubsection{Closed Form solutions for Alternating Minimization: Examples}\label{CF}

Here we focus on the alternating minimization case and discuss some settings in which it is possible to obtain a closed form solution for the procedures \textsc{SupervisedStep} and \textsc{UnsupervisedStep}.
 
\paragraph{(\textsc{SupervisedStep}) Least Square Loss.}
When the loss function $V$ is chosen to be least squares (i.e. $V(Y,Z) = \|Y-Z\|_F^2$ for any two matrices $Y,Z \in \mathbb{R}^{n \times m}$) and the structure matrix $A$ is fixed, a closed form solution for the coefficient matrix $C$ returned by the \textsc{SupervisedStep} procedure can be easily derived (see for instance~\cite{alvarez12}):
$$
vec(C)=(I_T \otimes K+\lambda A^{-1} \otimes I_n)^{-1}vec(Y).
$$
Here, the symbol $\otimes$ denotes the Kronecker product, while the notation $vec(M) \in \mathbb{R}^{nm}$ for a matrix $M\in\mathbb{R}^{n \times m}$ identifies the concatenation of its columns in a single vector. In~\cite{minh11} the authors proposed a faster approach to solve this problem in closed form based on Sylvester's method.

\paragraph{(\textsc{UnsupervisedStep}) $p$-Schatten penalties.}\label{sec:p-schatten_penalties}
We consider the case in which $F$ is chosen to be a spectral penalty of the form $F(\cdot) = \|\cdot\|_p^p$ with $p\geq1$. Also in this setting the optimization problem has a closed form solution, as shown in the following.

\begin{proposition}\label{prop:p_solution}
Let the penalty of problem~\eqref{eq:perturbation} be $F = \|\cdot\|_p^p$ with $p\geq1$. Then, for any  $C\in\mathbb{R}^{n \times T}$ fixed, the optimization problem~\eqref{eq:perturbation} in the block variable $A$ has a minimizer of the form
\begin{equation}\label{eq:p_solution}
A_C^{\delta} = \sqrt[p+1]{(C^\top K C + \delta^2 I_T)/\lambda}.
\end{equation}
\end{proposition}

Proposition~\ref{prop:p_solution} generalizes a similar result originally proved in in~\cite{argyriou08} for the special case $p=1$ and provides an explicit formula for the \textsc{UnsupervisedStep} of Algorithm~\ref{alg:bcd}. We report the proof in the supplementary material.

\section{Previous Work: Comparison and Discussion}\label{sec:examples}

The framework introduced in problem~\eqref{eq:nonconvex} is quite general and accounts for several choices of loss function and task-structural priors. Section~\ref{sec:learn_joint} has been mainly devoted to derive efficient and generic optimization procedures; in this section we focus our analysis on the modeling aspects, investigating the impact of different structure penalties on the multi-task learning problem. In particular, we will briefly review some multi-task learning method previously proposed, discussing how they can be formulated as special cases of problem~\eqref{eq:nonconvex} (or, equivalently, \eqref{eq:convex_equivalence}).

\paragraph{Spectral Penalties.} The penalty $F = \|\cdot\|_F^2$ was considered in~\cite{dinuzzo11}, together with a least squares loss function and the non convex problem~\eqref{eq:nonconvex} is solved directly by alternating minimization.
%The motivation for this approach was that, since the functional in~\eqref{eq:nonconvex} is invex on $\mathbb{R}^{n \times T} \times S_{++}^T$, any minimizer on $\mathbb{R}^{n \times T} \times S_{++}^T$ is actually a global minimizer. In fact, 
However, as pointed out in~Sec.~\ref{sec:learn_joint}, solving the non convex problem (although invex, see the discussion on Corollary~\ref{cor:invexity}) directly could in principle become problematic when the alternating minimization sequence gets close to the boundary of $\mathbb{R}^{n \times T} \times S_{++}^T$. A related idea is that of considering $F(A) = tr(A)$ (i.e. the $1$-Schatten norm). This latter approach can shown to be equivalent to the Multi-Task Feature Learning setting of~\cite{argyriou08} (see supplementary material).

\begin{figure*}
        \begin{center}
        
            \includegraphics[width=.24\textwidth]{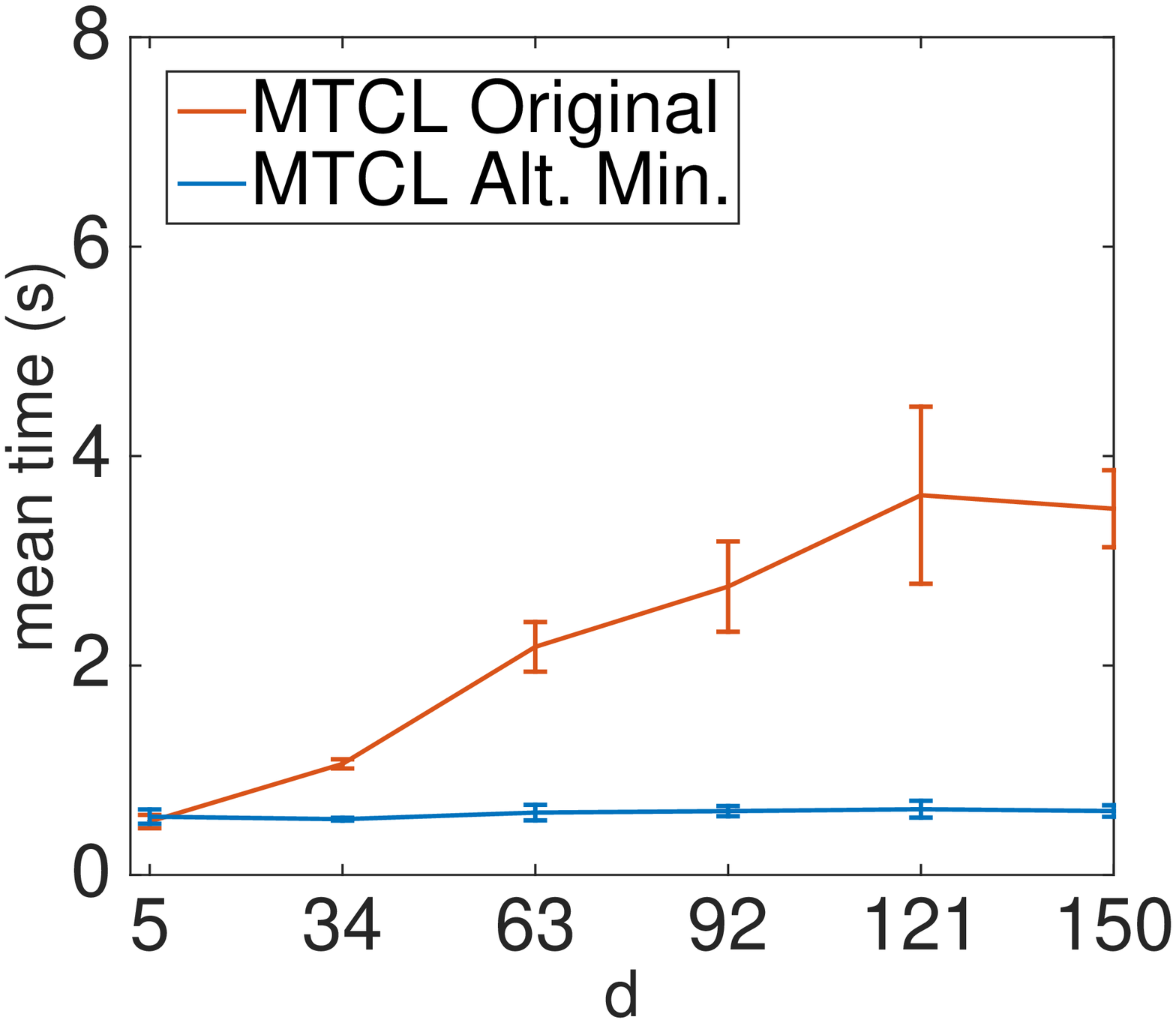}
            \includegraphics[width=.24\textwidth]{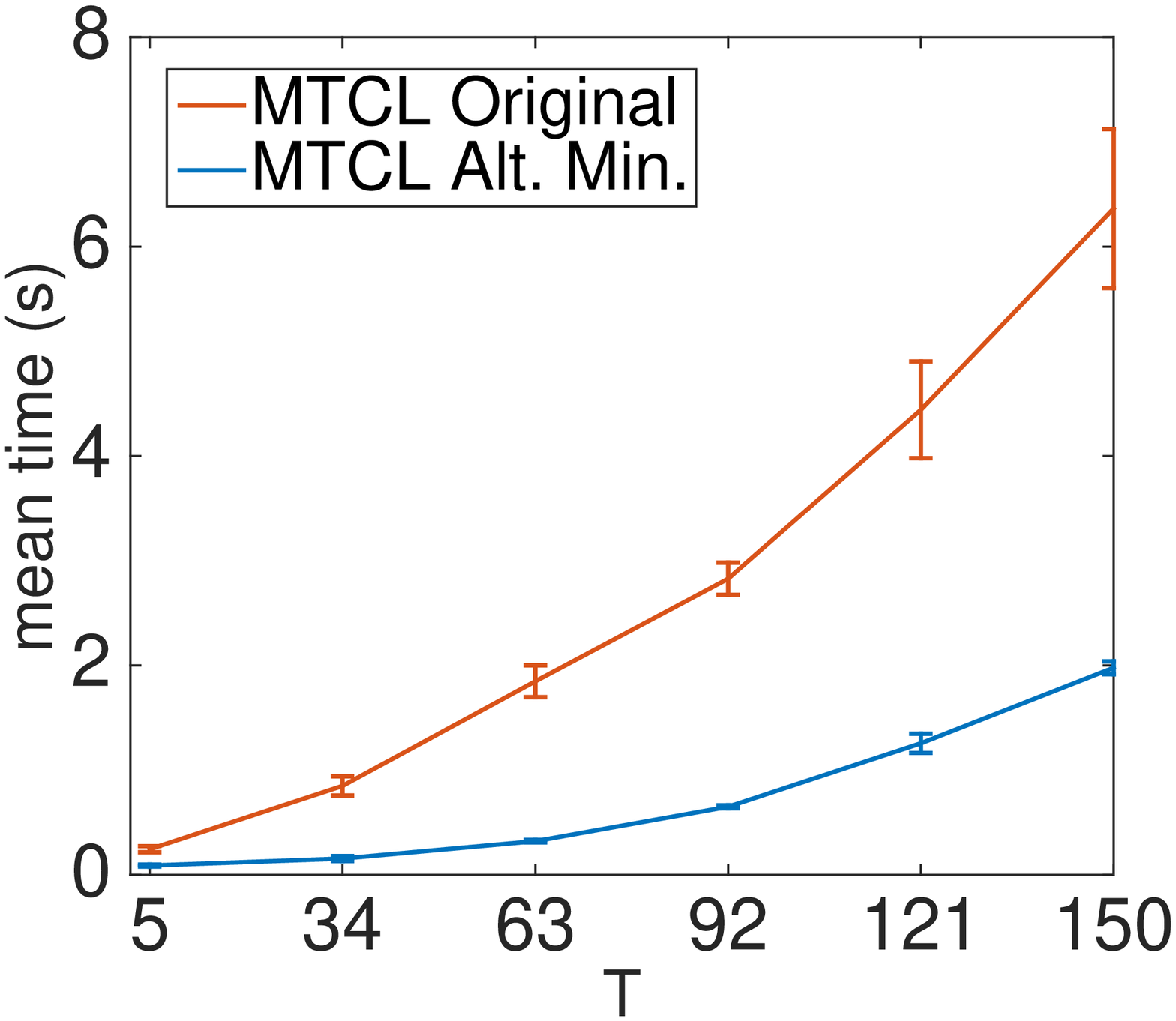}
            \includegraphics[width=.24\textwidth]{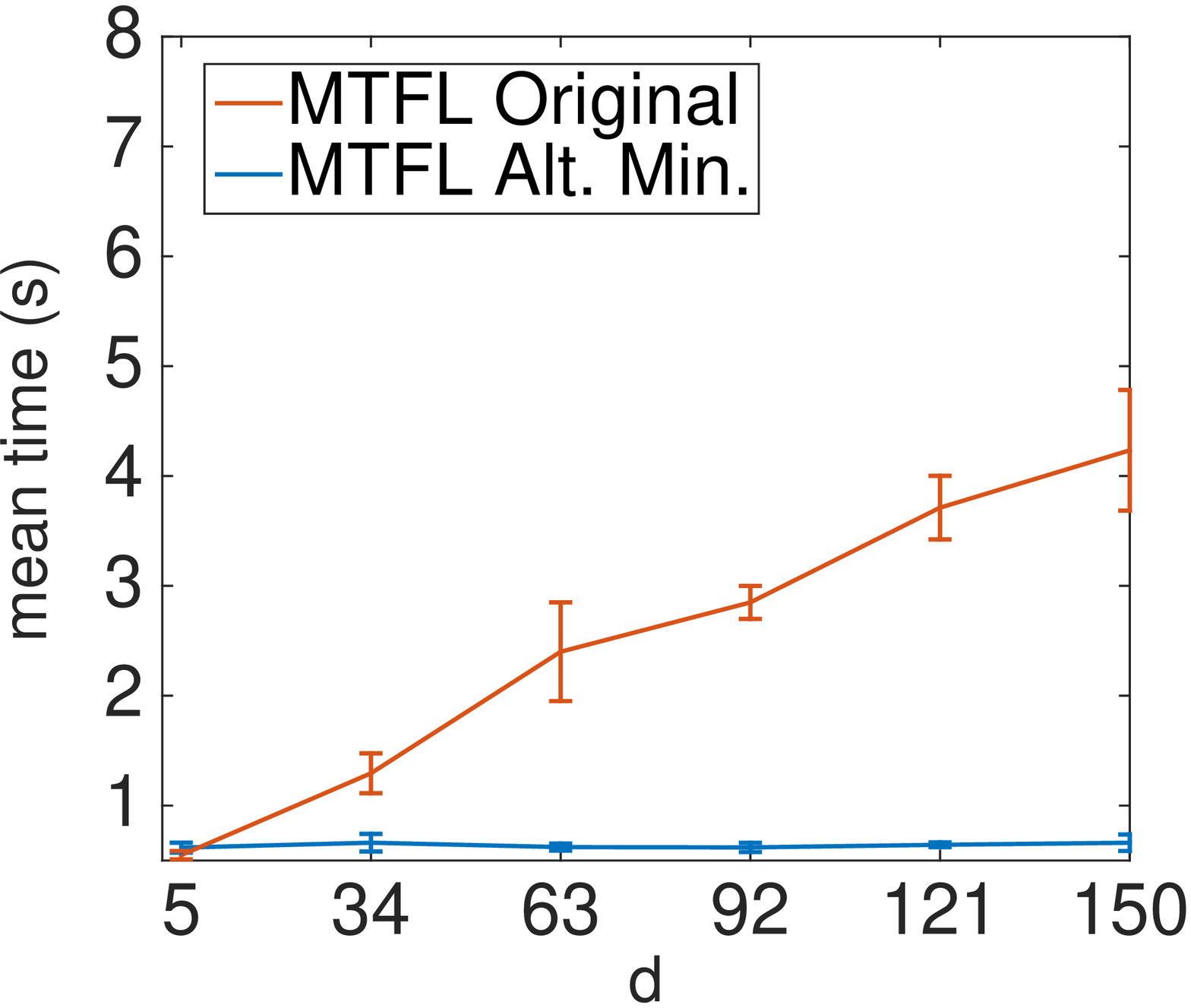}
            \includegraphics[width=.24\textwidth]{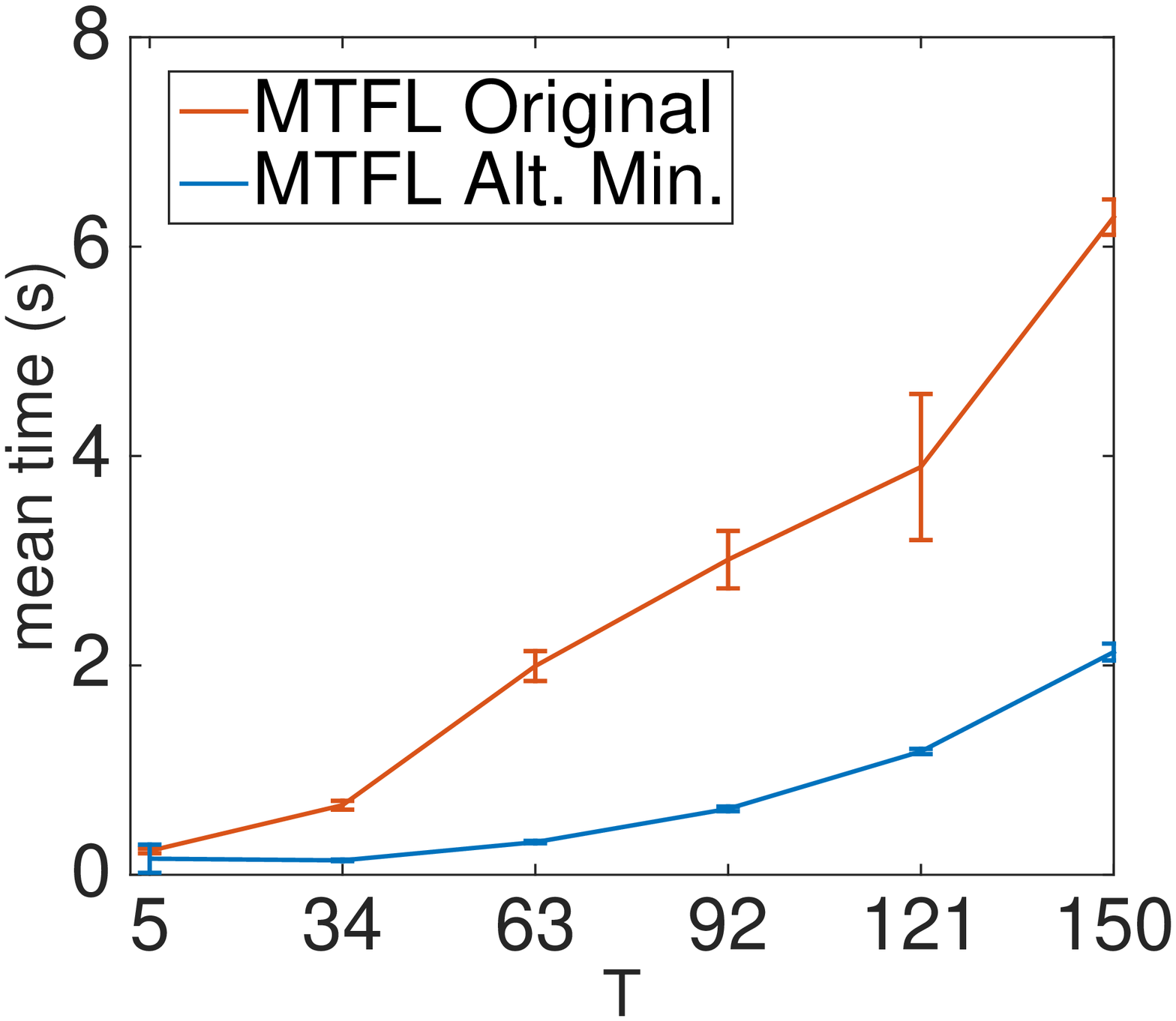}
        \end{center}
       \caption{Comparison of the computational performance of the alternating minimization strategy studied in this paper with respect to the optimization methods proposed for MTCL in~\cite{jacob08} and MTFL~\cite{argyriou08} in the original papers. Experiments are repeated for different number of tasks and input-space dimensions as described in Sec.~\ref{sec:speed}.}\label{fig:speed}%
\end{figure*}

%\paragraph{Nuclear Norm Regularization} \textcolor{blue}{[TO RE-EDIT]}
%
%A possible regularization setting for a multi-task learning problem is to penalize the different tasks to span a high dimensional subspace in the hypothesis space. \\
%
%
%In a linear setting where the coefficient matrix associated to the tasks is identified by $W \in\mathbb{R}^{d \times T}$, this regularization can be imposed by the nuclear norm $\|W\|_*=tr(\sqrt{W^\top W})$. From Corollary~\eqref{cor:p_solution}, this is clearly equivalent to choose $F(A) = tr(A)$ (i.e. the $1$-Schatten norm) in the setting of~\eqref{eq:perturbation}. Notice that this observation corresponds to an interesting result in~\cite{grave11} where the authors imposed the nuclear norm penalty on the feature space.
\paragraph{Cluster Tasks Learning.} In \cite{jacob08}, the authors studied a multi-task setting where tasks are assumed to be organized in a fixed number $r$ of unknown disjoint clusters. While the original formulation was conceived for linear setting, it can be easily extended to non-linear kernels and cast in our framework. Let $E\in\{0,1\}^{T \times r}$ be the binary matrix whose entry $E_{st}$ has value $1$ or $0$ depending on whether task $s$ is in cluster $t$ or not. Set $M=I - E^\dagger E^\top$, and $U=\frac{1}{T}11^{\top}$. In~\cite{jacob08} the authors considered a regularization setting of the form of~\eqref{eq:convex_equivalence} where the structure matrix $A$ is parametrized by the matrix $M$ in order to reflect the cluster structure of the tasks. More precisely:
$$A^{-1}(M)=\epsilon_{M}U+\epsilon_{B}(M-U)+\epsilon_{W}(I-M)$$
where the first term characterizes a global penalty on the average of all tasks predictors, the second term penalizes the between-clusters variance, and the third term controls the tasks variance within each cluster. Clearly, it would be ideal to identify an optimal matrix $A(M)$ minimizing problem~\eqref{eq:convex_equivalence}. However, $M$ belongs to a discrete non convex set, therefore authors propose a convex relaxation by constraining $M$ to be in a convex set $\mathcal{S}_c=\{ M \in S^{T}_{+} , 0\preceq M \preceq I, tr(M)=r \}$. In our notations $F(A)$ is therefore the indicator function over the set of all matrices $A = A(M)$ such that $M \in \mathcal{S}_c$. The authors propose a pseudo gradient descent method to solve the problem jointly.

\paragraph{Convex Multi-task Relation Learning.}
Starting from a multi-task Gaussian Process setting, in~\cite{zhang10}, authors propose a model where 
the covariance among the coefficient vectors of the $T$ individual tasks is controlled by a matrix $A\in S_{++}^T$ in the form of a prior. The initial maximum likelihood estimation problem is relaxed to a convex optimization with target functional of the form
\begin{equation}
\|Y - KC \|_F^2 + \lambda_1 \ tr(C^\top K C) + \lambda_2 \ tr(A^{-1} C^\top K C) 
\end{equation}
constrained to the set $\mathcal{A}=\{A \ | \ A\in S_{++}^T, tr(A)=1)$. This setting is equivalent to problem~\eqref{eq:convex_equivalence} (by choosing $F$ to be the indicator function of $\mathcal{A}$) with the addition of the term $tr(C^\top K C)$.

\paragraph{Non-Convex Penalties.}
Often times, interesting structural assumptions cannot be cast in a convex form and indeed several works have proposed non-convex penalties to recover interpretable relations among multiple tasks. For instance~\cite{argyriou13} requires $A$ to be a graph Laplacian, or~\cite{dinuzzo13} imposes a low-rank factorization of $A$ in two smaller matrices. In~\cite{mroueh11,kumar12} different sparsity models are proposed.\\
Interestingly, most of these methods can be naturally cast in the form of problem~\eqref{eq:nonconvex} or \eqref{eq:convex_equivalence}. Unfortunately our analysis of the barrier method does not necessarily hold also in these settings and therefore Alternating Minimization is not guaranteed to lead to a stationary point.

\section{Experiments}\label{sec:experiments}

We empirically evaluated the efficacy of the block coordinate optimization strategy proposed in this paper on both artificial and real datasets. Synthetic experiments were performed to assess the computational aspects of the approach, while we evaluated the quality of solutions found by the system on realistic settings.

\subsection{Computational Times}\label{sec:speed}

As discussed in Sec.~\ref{sec:examples}, several methods previously proposed in the literature, such as Multi-task Cluster Learning (MTCL) \cite{jacob08} and Multi-task Feature Learning (MTFL~\cite{argyriou08}]), can be formulated as special cases of problem~\eqref{eq:nonconvex} or \eqref{eq:convex_equivalence}. It is natural to compare the proposed alternating minimization strategy with the optimization solution originally proposed for each method. To assess the system's performance with respect to varying dimensions of the feature space and an increasing number of tasks, we chose to perform this comparison in an artificial setting.\\
We considered a linear setting where the input data lie in $\mathbb{R}^d$ and are distributed according to a normal distribution with zero mean and identity covariance matrix. $T$ linear models $w_t \in \mathbb{R}^d$ for $t=1,\dots,T$ were then generated according to a normal distribution in order to sample $T$ distinct training sets, each comprising of $30$ examples $(x_i^{(t)},y_i^{(t)})$ such that $y_i^{(t)} = \langle w_t, x_i^{(t)}\rangle + \epsilon$ with $\epsilon$ Gaussian noise with zero mean and $0.1$ standard deviation. On these learning problems we compared the computational performance of our alternating minimization strategy and the original optimization algorithms originally proposed for MTCL and MTFL and for which the code has been made available by the authors'. In our algorithm we used $A_0 = I$ identity matrix as initialization for the alternating minimization procedure. We used a least-squares loss for all experiments.\\
Figure~\ref{fig:speed} reports the comparison of computational times of alternating minimization and the original methods to converge to the same minima (of respectively the functional of MTCL and MTFL). We considered two settings: one where the number of tasks was fixed to $T=100$ and $d$ increased from $5$ to $150$ and a second one wher $d$ was fixed to $100$ and $T$ varied bewteen $5$ and $150$. To account for statistical stability we repeated the experiments for each couple $(T,d)$ and different choices of hyperparameters while generating a new random datasets at each time. We can make two observations from these results: 1) in the setting where $T$ is kept fixed we observe a linear increase in the computational times for both original MTCL and MTFL methods, while alternating minimization is almost constant with respect to the input space dimension. 2) When $d$ is fixed and the number of tasks increases, all optimization strategies require more time to converge. This shows that in general alternating minimization is a viable option to solve these problems and in particular, when $T << min(d,n)$ -- which is often the case in non-linear settings --this method is particularly efficient.

\begin{table*}[t]
\scriptsize
\begin{center}
%\rowcolors{3}{}{gray!35}
\hspace*{-2cm}            
    \begin{tabular}{l c >{\columncolor{gray!35}}c c >{\columncolor{gray!35}}c c >{\columncolor{gray!35}}c c >{\columncolor{gray!35}}c c}
    
        & \multicolumn{2}{c}{50 tr. samples per class} & \multicolumn{2}{c}{100 tr. samples per class} & \multicolumn{2}{c}{150 tr. samples per class} & \multicolumn{2}{c}{200 tr. samples per class} &  \tstrut \bstrut \\ 
        & {\bf nMSE ($\pm$ std)} & \cellcolor{white} {\bf nI} & {\bf nMSE ($\pm$ std)} & \cellcolor{white} {\bf nI} & {\bf nMSE ($\pm$ std)} & \cellcolor{white} {\bf nI} & {\bf nMSE ($\pm$ std)} & \cellcolor{white} {\bf nI} & \tstrut \bstrut \\ 

        \specialrule{.1em}{.05em}{.0em} 

        {\bf STL}       & $0.2436 \pm 0.0268$ & $0$ & $0.1723 \pm 0.0116$ & $0$ & $0.1483 \pm 0.0077$ & $0$ & $0.1312 \pm 0.0021$ & $0$     & \tstrut \bstrut \\ 

        {\bf MTFL}      & $0.2333 \pm 0.0213$ & $0.0416$ & $0.1658 \pm 0.0107$ & $0.0379$ & $0.1428 \pm 0.0083$ & $0.0281$ & $0.1311 \pm 0.0055$ & $0.0003$ & \tstrut \bstrut \\ 

        {\bf MTRL}      & $0.2314 \pm 0.0217$ & $0.0404$ & $0.1653 \pm 0.0112$ & $0.0401$ & $0.1421 \pm 0.0081$ & $0.0288$ & $0.1303 \pm 0.0058$ & $0.0071$ & \tstrut \bstrut \\ 

        {\bf OKL}       & $0.2284 \pm 0.0232$ & $0.0630$ & $0.1604 \pm 0.0123$ & $0.0641$ & $\mathbf{0.1410 \pm 0.0087}$ & $0.0350$ & $0.1301 \pm 0.0073$ & $0.0087$ & \tstrut \bstrut \\ 

    \end{tabular}

\end{center}
\caption{Comparison of Multi-task learning methods on the Sarcos dataset. The advantage of learning the tasks jointly decreases as more training examples became available.}
\label{tab:sarcos}
\end{table*}

\subsection{Real dataset}

We assessed the benefit of adopting multi-task learning approaches on two real dataset. In particular we considered the following algorithms: Single Task Learning (STL) as a baseline, Multi-task Feature Learning (MTFL)~\cite{argyriou08}, Multi-task Relation Learning (MTRL)~\cite{zhang10}, Output Kernel Learning (OKL) \cite{dinuzzo11}. We used least squares loss for all experiments.

\paragraph{Sarcos.}

Sarcos\footnote{url{http://www.gaussianprocess.org/gpml/data/}} is a regression dataset designed to evaluate machine learning solutions for inverse dynamics problems in robotics. It consists in a collection of $21$-dimensional inputs, i.e. the joint positions, velocities and acceleration of a robotic arm with $7$ degrees of freedom and $7$ outputs (the tasks), which report the corresponding torques measured at each joint.\\
For each task, we randomly sampled $50,100,150$ and $200$ training examples while we kept a test set of $5000$ examples in common for all tasks. We used a linear kernel and performed $5$-fold crossvalidation to find the best regularization parameter according to the normalized mean squared error (nMSE) of predicted torques. We averaged the results over $10$ repetitions of these experiments. The results, reported in Table~\ref{tab:sarcos}, show clearly that to adopt a multi-task approach in this setting is favorable; however, in order to quantify more clearly such improvement, we report in Table~\ref{tab:sarcos} also the {\it normalized improvement} (\textit{nI}) over single-task learning (STL). For each multi-task method MTL, the normalized improvement nI(MTL) is computed as the average
$$
\mbox{nI(MTL)} = \frac{1}{n_{exp}} \sum_{i=1}^{n_{exp}} \frac{\mbox{nMSE}_i(\mbox{STL})-\mbox{nMSE}_i(\mbox{MTL})}{\sqrt{\mbox{nMSE}_i(\mbox{STL})\cdot\mbox{nMSE}_i(\mbox{MTL})}}
$$
over all the $n_{exp} = 10$ experiments of the normalized differences between the nMSE achieved by respectively the STL approach and the given multi-task method MTL.

\begin{table}
\scriptsize
\begin{center}
\rowcolors{3}{}{gray!35}

\begin{tabular}{lcccccc}
    &  \multicolumn{6}{c}{\bf Accuracy (\%) per \# tr. samples per class}  \tstrut \bstrut \\
    & \multicolumn{2}{c}{$50$} & \multicolumn{2}{c}{$100$} & \multicolumn{2}{c}{$150$} \tstrut \bstrut \\
    % \toprule
    \specialrule{.1em}{.05em}{.0em}

    {\bf STL}                    & $72.23$ & $\pm 0.04$ & $76.61$ & $\pm 0.02$ & $79.23$ & $\pm 0.01$ \tstrut \bstrut \\

    {\bf MTFL}                   & $73.23$ & $\pm .08$ & $77.24$ & $\pm .05$  & $80.11$ & $\pm .03$  \tstrut \bstrut \\
                                                
    {\bf MTRL}                   & $73.13$ & $\pm 0.08$ & $77.53$ & $\pm 0.04$ & $80.21$ & $\pm 0.05$ \tstrut \bstrut \\
                                                
    {\bf OKL}                    & $72.25$ & $\pm 0.03$ & $77.06$ & $\pm 0.01$ & $80.03$ & $\pm 0.01$ \tstrut \bstrut \\
                                                
    % \multirow{2}{*}{\bf Sparse}                 & $\mathbf{73.50}$& $\mathbf{78.23}$ & $\mathbf{81.32}$ \tstrut \bstrut \\  
    %                                             & \cellcolor{gray!35} $\pm 0.11$& \cellcolor{gray!35} $\pm 0.06$ & \cellcolor{gray!35} $\pm 0.08$ \tstrut \bstrut \\  
    \end{tabular}
\end{center}
\caption{Classification results on the $15$-scene dataset. Four multi-task methods and the single-task baseline are compared.}
\label{tab:15scenes}
\end{table}

\paragraph{$15$-Scenes.}
$15$-Scenes\footnote{http://www-cvr.ai.uiuc.edu/ponce\_grp/data/} is a dataset designed for scene recognition, consisting in a $15$-class classification problem. We represented images using LLC coding~\cite{wang10} and trained the system on a training set comprising $50$, $100$ and $150$ examples per class. The test set consisted in $7500$ images evenly divided with respect to the $15$ scenes. Table~\ref{tab:15scenes} reports the mean classification accuracy on $20$ repetitions of the experiments. It can be noticed that while all multi-task approach seem to achieve approximately similar performance, these are consistently outperforming the STL baseline.

%and in particular that the Sparse method proposed in this paper further improves the results of previous methods. However the improvement provided by the multi-task approach decreases as more training examples become available.
% In order to quantify more clearly the advantages of adopting multi-task methods in the regression setting, we reported in Table~\ref{tab:sarcos} the {\it normalized improvement} (\textit{nI}) over single-task learning (STL). For each multi-task method MTL, the normalized improvement nI(MTL) is computed as the average
% $$
% \mbox{nI(MTL)} = \frac{1}{n_{exp}} \sum_{i=1}^{n_{exp}} \frac{\mbox{nMSE}_i(\mbox{STL})-\mbox{nMSE}_i(\mbox{MTL})}{\sqrt{\mbox{nMSE}_i(\mbox{STL})\cdot\mbox{nMSE}_i(\mbox{MTL})}}
% $$
% over all the $n_{exp} = 10$ experiments of the normalized differences between the nMSE achieved by respectively the STL approach and the given multi-task method MTL.

\section{Conclusions}

We have studied a general multi-task learning framework where the tasks structure can be modeled compactly in a matrix. For a wide family of models, the problem of jointly learning the tasks and their relations can be cast as a convex program, generalizing previous results for special cases~\cite{argyriou08,dinuzzo11}. Such an optimization can be naturally approached by block coordinate minimization, which can be seen as alternating between supervised and unsupervised learning steps optimizing respectively the tasks or their structure. We evaluated our method real data, confirming the benefit of multi-task learning when tasks share similar properties.\\
From an optimization perspective, future work will focus on studying the theoretical properties of block coordinate methods, in particular regarding convergence rates. Indeed, the empirical evidence we report suggests that similar strategies can be remarkably efficient in the multi-task setting. From a modeling perspective, future work will focus on studying wider families of matrix-valued kernels, overcoming the limitations of separable ones. Indeed, this would allow to account also for structures in the interaction space between the input and output domains jointly, which is not the case for separable models. 
%Up to our knowledge this investigation has been conducted only from a theoretical perspective~\cite{micchelli04}, but no practical applications employing non-separable kernels have been proposed.

%bibliography

{\small
\bibliographystyle{ieee}
\bibliography{biblio}
}

\clearpage

\section*{Appendix}

%%
% Unifying Framework 
%%

\section*{Imposing Known Structure on the Tasks}

\subsubsection*{Coding and Embedding}\label{sec:coding}

A common approach to encode knowledge of the tasks relations consists in mapping the output space $\mathcal{Y}^T$ in a new $\widetilde{\mathcal{Y}}\subseteq\mathbb{R}^\ell$ and then solve $\ell$ independent standard learning problems (e.g. RLS, SVM, Boosting, etc.~\cite{fergus10}) or a single one with a joint loss (e.g. Ranking~\cite{joachims09}) using the mapped outputs as training observation. The goal is to implicitly exploit the structure of the new space to enforce known (or desired) relations among tasks. 

The most popular setting for these \textit{embedding} (or \textit{coding}) methods is multi-class classification since in several realistic learning problems, classes can be organized in informative structures such as hierarchies or trees. Interestingly, due to the symbolic nature of the classes representation as canonical basis of $\mathbb{R}^T$, nonlinear embeddings are not particularly meaningful in classification contexts. Indeed the literature on coding methods for multi-task learning has been mainly concerned with the design of linear operators $L:\mathcal{Y}^T\to\widetilde{\mathcal{Y}}$ \cite{fergus10}. In the following we show that a tight connection exists between coding methods and our multi-task learning setting. 

For a fixed linear operator $L\in\mathbb{R}^{\ell \times T}$, we can solve the ``coded'' problem using the notation of~\eqref{eq:learning_problem_matrix} and a kernel of the form $\Gamma = k I_{\ell}$ with $I_\ell$ the $\ell\times\ell$ identity matrix (``independent tasks'' kernel)
\begin{equation}\label{eq:coded_problem}
\begin{aligned}
\underset{\widetilde{C}\in \mathbb{R}^{n \times \ell}}{\text{minimize}} \ \ V(\widetilde{Y}, K \widetilde{C}) + \lambda \ tr(\widetilde{C}^\top K \widetilde{C})\end{aligned}
\end{equation}

From the Representer theorem we know that the solution of~\eqref{eq:coded_problem} will have the form $f(x) = \sum_{i=1}^n k(x,x_i) \tilde{c}_i = \sum_{i=1}^n k(x,x_i) L c_i$, for some $c_i \in \mathbb{R}^T$ and $\tilde{c}_i = L c_i \in L(\mathbb{R}^T)$. Therefore, we can constrain~\eqref{eq:coded_problem} on matrices $\widetilde{C} = C L$ with $C\in\mathbb{R}^{n \times T}$, implying that the best solution for~\eqref{eq:coded_problem} belongs to the set of functions $f = L \circ g \in\mathcal{H}_{kI_\ell}$ with $g\in\mathcal{H}_{kI_T}$.

For those loss functions $\mathcal{L}$ that depend only on the inner product between the vectors of prediction and the ground truth (e.g. logistic or hinge \cite{joachims09,weston11}, see below), the ``coded'' Problem~\eqref{eq:coded_problem} on $\widetilde{\mathcal{Y}}$ with kernel $k I_{\ell}$ is equivalent to~\eqref{eq:learning_problem_matrix} on $\mathcal{Y}$ with kernel $kL^\top L$. More precisely, if the multi-output loss can be written so that $\mathcal{L}(\tilde{y},f(x)) = \mathcal{L}(\langle \tilde{y}, f(x)\rangle_{\widetilde{\mathcal{Y}}})$ for all $\tilde{y}\in\widetilde{\mathcal{Y}}$ and $x\in\mathcal{X}$, we have
\begin{equation}\label{eq:coding_inner_product}
\langle \tilde{y},f(x)\rangle_{\widetilde{\mathcal{Y}}} = \langle L y, L g(x) \rangle_{\widetilde{\mathcal{Y}}} = \langle y, L^\top L g(x)\rangle_{\mathcal{Y}}
\end{equation}
where $y\in\mathcal{Y}$ is such that $L y = \tilde{y}$ and $L^\top$ denotes the adjoint operator of $L$ (in this case just the transpose matrix since $L$ is a linear operator between vector spaces over the real field). Therefore, the two terms in the functional of~\eqref{eq:coded_problem} become
$$
V(\widetilde{Y},K\widetilde{C}) = V(YL^\top,KCL^\top) = V(Y,KCL^\top L)
$$
where the last equality makes use of the property in eq.~\eqref{eq:coding_inner_product}, and
$$
tr(\widetilde{C}^\top K \widetilde{C}) = tr(LC^\top K CL^\top) = tr(L^\top L C^\top K C)
$$
proving the aforementioned equivalence between Problems~\eqref{eq:coded_problem} and~\eqref{eq:learning_problem_matrix} by choosing $A = L^\top L$.

\paragraph{Semantic Label Sharing}
In~\cite{fergus10} the authors proposed a strategy to solve a large multi-class visual learning problem that exploited the semantic information provided by the WordNet~\cite{fellbaum98} to enforce specific relations among tasks. In particular, by designing a ``semantic'' distance between classes using the WordNet graph, the authors were able to generate a similarity matrix $L\in S_+^T$ encoding the most relevant class relations. They used this matrix to map the original outputs (i.e. the canonical basis of $\mathbb{R}^T$) into a new basis where euclidean distances between output codes would reflect the semantic ones induced by the WordNet priming. Then they applied a semi-supervised One-Vs-All approach on the new output space.

\subsubsection*{Output Metric}\label{sec:metric}

In multi-output settings, another approach to implicitly model the tasks relations consists in changing the metric on the output space $\mathbb{R}^T$. In particular, we can define a matrix $\Theta \in S_+^T$ and denote the induced inner product on $\mathbb{R}^T$ as $\langle y,y' \rangle_{\Theta} = \langle y,\Theta y' \rangle_{\mathbb{R}^T}$ for all $y,y'\in\mathbb{R}^T$. For loss functions $\mathcal{L}$ such as those mentioned in Sec.~\ref{sec:coding} (e.g. hinge, logistic, etc.) that depend only on the inner product between observations and predictions, we have that for a fixed $\Theta$ the new loss is defined as $\mathcal{L}_{\Theta}(y,f(x)) = \mathcal{L}(\langle y, f(x)\rangle_{\Theta})=\mathcal{L}(\langle y, \Theta f(x)\rangle_{\mathbb{R}^T})$ and induces a learning problem of the form
\begin{equation}\label{eq:deformed_problem}
\begin{aligned}
\underset{C\in \mathbb{R}^{n \times T}}{\text{minimize}} \ \ V(\widetilde{Y}, K C \Theta) + \lambda \ tr(\Theta C^\top K C)
\end{aligned}
\end{equation}
which is clearly equivalent to solving~\eqref{eq:learning_problem_matrix} choosing the kernel $k\Theta$. Notice that the second term in eq.~\eqref{eq:deformed_problem} derives from the observation that with the new metric, the norm in the RKHSvv becomes $\|f\|_{kI_{T}}^2 = \langle f,f \rangle_{kI_{T}} = \sum_{i,j}^n\sum_{t,s}^T k(x_i,x_j) \langle c_t, c_s\rangle_\Theta = tr(\Theta C^\top K C)$ as required.

\paragraph{metric learning}
In~\cite{lozano10} the authors proposed a metric learning framework in which both the new metric $A$ (or $\Theta$) and the task predictors were estimated simultaneously. Adopting almost the same notation of Problem~\eqref{eq:nonconvex}, they used the least squares loss and imposed a penalty $F(A) = -log(det(A))$ on the metric/structure matrix. A further penalty was also imposed on $A$, in order to enforce specific sparsity patterns. The only difference with our framework is that in~\cite{lozano10} the authors do not impose the regularization term $tr(AC^\top K C)$. Notice however that such term allows us to apply Theorem~\ref{teo:convex_equivalence} and thus obtain the equivalence between~\eqref{eq:nonconvex} and \eqref{eq:convex_equivalence}. This is extremely useful from the optimization perspective since, for instance, for the least squares loss and log-determinant penalty mentioned above, Problem~\eqref{eq:convex_equivalence} is actually convex jointly, which is not the case for the framework in~\cite{lozano10}.

\section*{Learning the tasks and their structure}

\subsection*{Equivalence with the convex problem}

We will make use of the following observation

\begin{lemma}\label{lemma:range_equality}
Consider $K\in S_+^T$ and $C\in\mathbb{R}^{n \times T}$. Then $\ran(C^\top K C) = \ran(C^\top \sqrt{K}) = \ran(C^\top K)$.
\end{lemma}

\begin{proof}

The second equivalence follows directly from the observation that $C^\top K = (C^\top \sqrt{K})\sqrt{K}$ and  $C^\top \sqrt{K} = C^\top K (\sqrt{K})^\dagger$. Regarding the first equivalence, recall that for any $M \in \mathbb{R}^{T \times n}$, $\mathbb{R}^T = \ran(M) \oplus \nul(M)$, with $\nul(M)$ denoting the null space of $M$. Therefore we can alternatively prove that $\nul(C^\top K C) = \nul(C^\top \sqrt{K})$. Notice that clearly $\nul(C^\top \sqrt{K})\subseteq\nul(C^\top K C)$. Now, let $x\in\nul(C^\top K C)$ so that $0 = x^\top C^\top K C x = x^\top (\sqrt{K} C)^\top (\sqrt{K} C) x$. This implies that $x$ is a singular vector of $(\sqrt{K}C)$ with singular value equal to zero and therefore $x \in \nul(C^\top \sqrt{K})$.
\end{proof}

\begin{proof}\textit{(Theorem~\ref{teo:convex_equivalence})}

We need to prove that $\mathcal{C}$ is a convex set and that $tr(A^\dagger C^\top K C)$ is jointly convex on $\mathcal{C}$. Regarding the first part, notice that for $A \in S_+^T$ and $C\in\mathbb{R}^{n \times T}$ the constraint $\ran(C^\top KC) \subseteq\ran(A)$ can be equivalently rewritten as $\nul(C^\top KC)\supseteq\nul(A)$. Therefore, using Lemma~\ref{lemma:range_equality}, we can check convexity of $\mathcal{C}$ by showing that for any arbitrary couple $(A_1,C_1),(A_2,C_2)\in\mathcal{C}$ and any $\theta\in [0,1]$ we have $\nul(\theta A_1 + (1-\theta) A_2)\subseteq\nul(\theta C_1^\top K+ (1-\theta) C_2^\top K)$. Let us consider an arbitrary $x\in\nul(\theta A_1 + (1-\theta) A_2)$. We have
$$
0 = x^\top (\theta A_1 + (1-\theta)A_2) x = \theta x^\top A_1 x + (1-\theta) x^\top A_2 x.
$$
Since both $A_1$ and $A_2$ are PSD, the terms $x^\top A_i x$ are necessarily non-negative for both $i=1,2$. Hence, from the equation above we have $x^\top A_i x=0$, which is equivalent to $x \in \nul(A_1) \cap \nul(A_2) \subseteq \nul(C_1^\top K) \cap \nul(C_2^\top K)$. This means that $x$ is in the nullspace of both $C_1^\top K$ and $C_2^\top K$ and therefore also in the nullspace of any linear combination of the two. In particular $x\in\nul(\theta C_1^\top K + (1-\theta) C_2^\top K)$.

The proof for the convexity of $tr(A^\dagger C^\top K C)$ has been already pointed out elsewhere (see for instance~\cite{argyriou07}). For completeness, we provide an simpler derivation of this result which makes use of a Schur's complement argument and simple algebraic properties in line with~\cite{dinuzzo11} to show that the epigraph of the function is convex. Consider $A\in S_+^T$ and $C\in\mathbb{R}^{n \times T}$. From simple properties of the trace we have the equivalence $tr(A^\dagger C^\top KC) = vec(\sqrt{K}C)^\top (A^\dagger \otimes I_T) vec(\sqrt{K}C)$, where $\otimes$ identifies the Kronecker product and by $vec(\cdot)$ we denote the vectorization operator mapping a matrix $M\in\mathbb{R}^{n \times m}$ to the concatenation of all its columns $vec(M)\in\mathbb{R}^{nm}$. Since $\ran(A)\supseteq\ran(C^\top K C) = \ran(C\sqrt{K})$ we can apply the generalized Schur's complement to write the epigraph of $f(A,C) = tr(A^\dagger C^\top K C)$ as
\begin{equation*}
\begin{aligned}
epi \ f = \left\{ (t,A,C) \ \left| \ t \geq tr(A^\dagger C^\top KC) = \right. \right. \\
\left. vec(C\sqrt{K})^\top (A^\dagger \otimes I_T) vec(C\sqrt{K}), (A,C) \in\mathcal{C} \right\}= \\
= \left\{(t,A,C) \ \left| \ \left(\begin{array}{cc}A \otimes I_T  & vec(C\sqrt{K}) \\ vec(C\sqrt{K})^\top & t \end{array} \right) \succeq 0, \right. \right. \\
\left. (A,C)\in\mathcal{C} \right\}
\end{aligned}
\end{equation*}

where we write $X \succeq Y$ for any two symmetric matrices $X,Y\in S^m$ if and only if $X-Y\in S_+^m$. Notice that the block components of the matrix in the equation above are all linear with respect to $A,C$ and $t$ and therefore the convexity of $epi \ f$ follows by directly observing that for any couple $(t_1,A_1,C_1),(t_2,A_2,C_2)\in epi \ f$, the PSD constraint holds for any convex combination of the two.

We finally prove that the mapping between minimizers stated in Theorem~\eqref{teo:convex_equivalence}. First notice that for any $(C,A)\in\mathbb{R}^{n \times T} \times S_+^T$ we have $Q(C,A)=R(CA,A)$, with $(CA,A) \in dom R$ since clearly $\ran(A)\supseteq \ran(AC^\top K CA)$. Therefore $inf \ \{Q(C,A) \ | \ C\in\mathbb{R}^{n \times T}, A \in S_+^T \} \geq inf \ \{R(C,A) \ | \ (C,A)\in\mathcal{C} \}$. Analogously, given a point $(C,A)\in\mathcal{C}$ we have that $R(C,A)=R(CA^\dagger A,A)$ since $\ran(C^\top K)\subseteq\ran(A)$ and thus $V(y,KCAA^\dagger) = V(y,KC)$. Therefore $R(C,A)=R(CA^\dagger A,A)=Q(CA^\dagger,A)$, implying that $ inf \ \{R(C,A) \ | \ (C,A)\in\mathcal{C} \} \geq inf \ \{Q(C,A) \ | \ C\in\mathbb{R}^{n \times T}, A \in S_+^T \}$ and concluding the proof.
\end{proof}

%%
% Alternate Minimization
%%
\subsection*{A Barrier Method to Optimize~\eqref{eq:convex_equivalence}}

\begin{proof} \textit{(Theorem~\ref{teo:perturbation})}
To prove the existence of finite minimizers we need to show that there exists a minimizing sequence for $S^\delta$ such that it converges to a point in $dom S^\delta = \mathbb{R}^{n \times T} \times S_{++}^T$. To see this, consider a generic minimizing sequence, i.e. a sequence $\{(C_n,A_n)\}_{n\in\mathbb{N}}\subset dom S^\delta$ such that $S^\delta(C_n,A_n) \to inf_{C,A} S^\delta(C,A)$. Notice that we can separate $C_n$ in $C_n = \widehat{C}_n, + C_n^\perp$ with $\widehat{C}_n\in\ran(K)$ the range of the Gram matrix $K$ and $C_n^\perp\in\nul(K)$ its nullspace and that therefore $S^\delta(\widehat{C}_n,A_n) = S^\delta(C_n,A_n)$. This implies that the sequence $(\widehat{C}_n,A_n)$ is bounded, since, if it was not, we would have the coercive penalty $F$ or the $tr(A_n^{-1} \widehat{C}_n^\top K \widehat{C}_n)$ to go to infinity as $n$ grows. But this is not possible since $S^{\delta}(\widehat{C}_n,A_n) \to inf_{C,A} S^\delta(C,A) < +\infty$. Therefore  $(\widehat{C}_n,A_n)$ admits a converging subsequence. Suppose without loss of generality that $(C_n,A_n)$ converges to a point $(C^*,A^*) \in \overline{dom S^\delta} = \mathbb{R}^{n \times T} \times S_+^T$. We want to show that $(C^*,A^*)$ is actually in the $dom S^\delta = \mathbb{R}^{n \times T} \times S_{++}^T$, i.e. that $A^*$ is positive definite. But this is obvious since $\delta > 0$ and therefore if the $A_n$ were to converge to a point in $S_+^T \backslash S_{++}^T$, we would have that $\delta^2 \  tr(A_n^{-1}) \to +\infty$ and therefore $S^{\delta}(\widehat{C_n},A_n) \to + \infty$ as $n \to+\infty$. Finally, by the continuity of $S^\delta$, we have $S^\delta(\widehat{C}_n,A_n) \to S^\delta(C^*,A^*)$, therefore proving that $(C^*,A^*) \in \argmin_{C,A} S^\delta(C,A)$.

The second part of the proof requires the following preliminary steps:
\begin{enumerate}
\item $min_{C,A} R(C,A) = inf_{A,C} S^0(C,A)$ and they have same infimizers.
\item $g(\delta) = inf_{A,C} S^\delta(C,A)$ is continuous (in fact convex) with minimum in 0. 
\end{enumerate}

We prove the first point in Lemma~\ref{lemma:equal_inf}, while the second observation follows from the fact that the function $g$ is the point-wise infimum of a jointly convex function over a convex set. This requires to show that $\delta^2 tr(A^{-1})$ is jointly convex which follows the same reasoning as for the convexity of $tr(A^{-1}C^\top K C)$ in Theorem~\eqref{teo:convex_equivalence}.

Let us consider two sequences $\delta_n>0$ and $\{(C_n,A_n)\}_{n\in\mathbb{N}} \subset dom S^\delta = \mathbb{R}^{n \times T} \times S_{++}^T$ satisfying the hypothesis of the Theorem, i.e. $S^{\delta_n}(C_n,A_n) = min_{C,A} S^{\delta_n}(C,A)$. We will first prove the result for $C_n$ in the range of the Gram matrix $K$. Notice that under this requirement, the $(C_n,A_n)$ are bounded, since, analogously as for the proof above, if they were not we would have the coercive penalty $F$ or the $tr(A_n^{-1} C_n^\top K C_n)$ to go to infinity as $n$ grows. But this is not possible since $S^{\delta_n}(C_n,A_n) \to g(0) < +\infty$. Therefore, by points 1. and 2., $g(0) = min_{C,A}R(C,A)$ and the limit points of $(C_n,A_n)$ are minimizers for $R$. This finally implies that there exists a sequence $\{(C_n^*,A_n^*)\}_{n\in\mathbb{N}}\subseteq argmin_{C,A} R(C,A)$ such that $\|C_n-C_n^*\|_F+\|A_n-A_n^*\|_F$ tends to zero as $n$ goes to infinity. To see this, suppose by contradiction that it is not true and that there exists a subsequence $\{(C_{n_k},A_{n_k})\}_{k\in\mathbb{N}}$ and an $M>0$ such that $\|C_{n_k}-C^*\|_F + \|A_{n_k}-A^*\|_F > M$ for all $k>0$ and for all $(C^*,A^*) \in \argmin_{C,A} R(C,A)$. Now, since $(C_{n_k},A_{n_k})$ is a subsequence of $(C_n,A_n)$, we have that: $(i)$ $(C_{n_k},A_{n_k})$  is bounded (hence admits a converging subsequence) and $(ii)$ every converging subsequence tends to a minimizer of $R$. This clearly contradicts the hypothesis.

Now, consider the general case in which $C_n$ is not in the range of $K$: notice that similarly as before, $C_n$ can be separated in $C_n = \widehat{C}_n + C_n^\perp$ with $\widehat{C}_n \in \ran(K)$ the range of $K$ and $C_n^\perp\in\nul(K)$ its nullspace. Clearly, $S^{\delta_n}(\widehat{C}_n,A_n)=S^{\delta_n}(C_n,A_n)\to g(0)$ and therefore, from the discussion above we have a sequence $\{(\widehat{C}_n^{*},A_n^*)\}_{n\in\mathbb{N}} \subseteq \argmin_{C,A} R(C,A)$ such that $\|\widehat{C}_n-\widehat{C}_n^{*}\|_F+\|A_n-A_n^*\|_F \to 0$ as $n\to+\infty$. We can now observe that the sequence $(C_n^*,A_n^*) = (\widehat{C}_n^* + C_n^\perp,A_n^*)$ satisfies the statement of the Theorem: indeed $(i)$ the $(C_n^*,A_n^*)$ are minimizers for $R$ since $R(C_n^*,A_n^*) = R(\widehat{C}_n^*,A_n^*)$ and $(ii)$ $\|C_n-C_n^*\|_F = \|\widehat{C}_n-\widehat{C}_n^*\|_F \to 0$ for $n\to+\infty$. 
\end{proof}

\begin{lemma}\label{lemma:equal_inf}
$min_{A,C} R(C,A) = inf_{A,C} S^0(C,A)$ and they have same infimizers:
\end{lemma}

\begin{proof}
This fact follows from the observation that for all $\delta>0$, $dom S^\delta= dom S^0$ is equal to the interior of $dom R$ and that all minimizers for $R$ belong to $dom R$. To show this second statement we will prove that for any sequence $\{(C_n,A_n)\}_{n \in \mathbb{N}}\subset dom R$ and converging to some point $(\bar{C},\bar{A})\in\mathbf{R}^{n \times T} \times S_+^T \ \backslash \ dom R$, we have that $R(C_n,A_n) \to +\infty$ as $n$ goes to infinity. For simplicity of notation let us denote $\bar{B}=\bar{C}^\top K \bar{C}$ and analogously $B_n = C_n^\top K C_n$. Since from hypothesis $\ran(\bar{A})\not \supseteq \ran(\bar{C}^\top K \bar{C})$ we have that $\nul(\bar{A}) \not \subseteq \nul(\bar{B})$, or, in other words, there exists an eigenvector $\bar{v}$ for $\bar{A}$ such that $v\in\nul(A)$ and $\|\bar{B}\bar{v}\|_2 > 0$. 

Since the sequence $A_n$ converges to $\bar{A}$, we can identify a sequence of eigenvectors $v_n$ for $A_n$ such that $v_n \to \bar{v}$ and their associated eigenvalue $\lambda_n \to 0$ as $n$ goes to infinity. Notice that we can assume without loss of generality that $\lambda_n > 0$ for all $n$ since $\lambda_n =0$ would imply $v_n \in \nul(A_n) \subseteq \nul(B_n)$ but we have from hypothesis that $\|B_n v_n\|_2 \to \|\bar{B}\bar{v}\| > 0$. Therefore we have
$$
tr(A_n^\dagger B_n) \geq \lambda_n^{-1} v^\top_n B_n v_n = \lambda_n^{-1} \|B_n v_n\|_2^2 \to + \infty
$$

as $n$ goes to infinity.
\end{proof}

%% Spectral Regularization - p-Schatten norms
\subsection*{Spectral Regularization}

Proposition~\ref{prop:p_solution} follows directly from the following result

\begin{proposition}\label{prop:alignment}
Let $A,M\in S_+^n$ with $\ran(A) \supseteq \ran(M)$, $rank(M)=r$. Let $M=U \Sigma U^\top$ be an eigendecomposition of $M$ with $U \in O^n$ and $\Sigma \in S_+^n$ a diagonal matrix with eigenvalues in decreasing order. Then, there exists a matrix $A_* = U \Gamma U^\top \in S_+^n$ with $\Gamma \in S_+^n$ diagonal with $\Gamma_{i,i}=0$ $\forall i<r$, such that
\begin{equation}\label{eq:alignment}
tr(A^\dagger_* M) = tr(A^\dagger M) \text{ \ \ \ \ and \ \ \ \ } \|A_*\|_p \leq \|A\|_p \text{ \ \ } \forall p\geq1
\end{equation}
with the equality holding if and only if $A_*=A$.
\end{proposition}

\begin{proof}

To keep the notation uncluttered we prove the result for $\Theta=A^\dagger$. Consider an eigendecompositionn $\Theta=S \Lambda S^\top$ with $S\in O^n$ and $\Lambda\in S_+^n$ diagonal with eigenvalues in decreasing order. Let us define $R=U^\top S\in O^n$. Then 
$$
tr(\Theta M) = tr(R \Lambda R^\top \Sigma) = \sum_{i=1}^r \sigma_i \sum_{j=1}^n R_{ij}^2 \lambda_j = \sum_{i=1}^r \sigma_i \gamma_i
$$
where $\sigma_i$ and $\lambda_i$ are respectively the $i$-th eigenvalues of $M$ and $\Theta$ and we have defined $\gamma_i = \sum_{j=1}^n R_{ij}\lambda_j$ for $i \leq r$ and $\gamma_i=0$ otherwise. Hence, if we consider a diagonal matrix $\Gamma\in S_+^n$ such that $\Gamma_{ii}=\gamma_i$ and set $\Theta'=U\Gamma U^\top$ we obtain the left equivalence of eq.~\eqref{eq:alignment}, namely $tr(\Theta M) = tr(\Theta' M)$. Now, consider the $p$-Schatten norm of $\Theta'$
$$
\| (\Theta')^\dagger \|_p = \left(\sum_{i=1}^r \frac{1}{\gamma_i^p}  \right)^{1/p} = \left( \sum_{i=1}^r \frac{1}{\left( \sum_{j=1}^n R_{ij}^2 \lambda_j \right)^p}\right)^{1/p}.
$$
Notice that $R_{ij}=U_i^\top \cdot S_j$ corresponds to the projection of the $i$-th eigenvector of $M$ on the $j$-th eigenvector of $\Theta$. Since $\ran(\Theta)=\ran(A)\supseteq\ran(M)$, for any eigenvector $s\in\mathbb{R}^n$ in the nullspace of $\Theta$ (i.e. with associated eigenvalue $\lambda=0$), we have that $U_i^\top \cdot s=0$ for all $i \leq r$. Hence, $\forall i \leq r$, $1=R_i^\top \cdot R_i = \sum_{j=1}^n R_{ij}^2 = \sum_{j=1}^k R_{ij}^2$, where $k=rank(A)$. Therefore, since the $R_{ij}^2$s add up to $1$ and the scalar function $(1/x)^p$ is convex in $x\in\mathbb{R}_++$, we have
\begin{equation*}
\begin{aligned}
\sum_{i=1}^r \frac{1}{\left( \sum_{j=1}^n R_{ij}^2 \lambda_j \right)^p} \leq \sum_{i=1}^r \sum_{j=1}^k R_{ij}^2 \frac{1}{\lambda_j^p} \leq \\ 
\leq \sum_{j=1}^k \frac{1}{\lambda_j^p} \sum_{i=1}^n R_{ij}^2 = \sum_{j=1}^k \frac{1}{\lambda_j^p} = \|\Theta^\dagger\|_p^p
\end{aligned}
\end{equation*}

where we have made use of the fact that for all $j=1,\dots,n$ we have $\sum_{i=1}^n R_{ij} = R_j^\top \cdot R_j = 1$. Therefore, $\|(\Theta')^\dagger\|_p \leq \|\Theta^\dagger\|_p$. By taking $A'=(\Theta')^\dagger$ we have the desired result.
\end{proof}

Applied to the minimization in problem~\eqref{eq:convex_equivalence} with $C\in\mathbb{R}^{n \times T}$ fixed and $p$-Schatten penalty, Proposition~\ref{prop:alignment} states that a minimizer $A_C\in S_+^T$ has the same system of eigenvalues as $C^\top K C$ and their spectrum have same sparsity pattern (i.e. $\ran(C^\top K C)=\ran(A)$). This observation leads directly to the closed formula to find a $A_*$ stated in Proposition~\ref{prop:p_solution}.

\begin{proof} \textit{(Proposition~\ref{prop:p_solution})}
Consider the eigendecomposition $C^\top K C = M = U \Sigma U^\top$  with $U\in O^T$ and $\Sigma\in S^T_+$ diagonal with the eigenvalues arranged in descending order. We apply Proposition~\ref{prop:alignment} and obtain the minimizer $A_*=U \Gamma U^\top$ for $\Gamma \in S_+^T$ diagonal with same sparsity pattern as $\Sigma$. We can rewrite the target function as
$$
\sum_{t=1}^{r} \frac{\sigma_t}{\gamma_t} + \lambda \ \gamma_t.
$$ 
where $r=rank(M)$. Therefore, the optimization problem consists in minimizing the target function above with respect to the $\gamma_t$s. This is an unconstrained convex optimization of a differentiable coercive function bounded below and therefore it is sufficient to set the gradient to zero and solve with respect to the $\gamma_t$. It is clear that for each $t=1\dots r$, the minimizer is of the form $\gamma_t = \sqrt[p+1]{\sigma_t/\lambda}$, leading to the desired solution. 
\end{proof}

\section*{Linear Multi-task Learning}

Several works in multi-task learning have  focused on linear models where the multi-output predictor $f:\mathbb{R}^d \to \mathbb{R}^T$ is  parameterized by a matrix $W \in \mathbb{R}^{d \times T}$ whose columns $w_t \in \mathbb{R}^d$ are associated to the individual task-predictors $f_t(x) = \langle w_t, x\rangle_{\mathbb{R}^d}$ for any $x \in\mathbb{R}^d$. In this tasks structure can be imposed considering suitable matrix penalty  $\Omega: \mathbb{R}^{d \times T} \to \mathbb{R}$ and regularization schemes of form
\begin{equation}\label{eq:linear_problem_matrix}
\underset{W\in\mathbb{R}^{d \times T}}{\text{min.}} \ \ V(Y,XW) + \Omega(W)
\end{equation}
where $X \in \mathbb{R}^{n \times d}$ is the matrix whose rows correspond to the (transposed) input points in the training sets, ordered accordingly to the order in $Y$ \footnote{Again $V$ would weight with zeros the loss associated to entries for which examples are not available during training}. 
We can recognize two main classes of penalty functions.  A first class correspond to  methods that impose structured sparsity on the input features across the multiple tasks, for instance considering the penalty $\Omega(\cdot)= \|\cdot  \|_{2,1}$~\cite{argyriou08}, which encourages whole rows of $W$ to be  simultaneously sparse, see also 
\cite{jayaraman14,zhong12}.  A second class corresponds to spectral regularization methods defined by    penalties $\Omega$  acting on the singular values of $W$. Examples in this class include methods that impose low-rank assumptions~\cite{argyriou08} on the tasks, or search after tasks-cluster structures~\cite{jacob08}. Ideas related to a combination of the above methods can also be considered \cite{chen12}.

% If we choose a linear kernel,  the regularized approach in~\eqref{eq:convex_equivalence} can compared to the one described above. Indeed,  as shown in the supplementary material,  the spectral regularization of the form~\eqref{eq:linear_problem_matrix}  defined by $p$-Schatten norms of $W$ can be recovered as special cases in our framework.  In this view,  for this class of models we provide an equivalent formulation which is potential more appealing form a computational perspective. Moreover, such a formulation seamlessly generalize to non linear input kernels.

Most Linear multi-task learning problems of the form~\eqref{eq:linear_problem_matrix} with $\Omega$ spectral penalty, can be formulated in terms of problem~\eqref{eq:convex_equivalence} for a suitable choice of $F$. Indeed it can be shown that for several spectral norms, such as the p-schatten norms, the penalty $\Omega$ can be written as
$$
\Omega(W) = \underset{A \in S^T_{++}}{inf} \ trace(WA^{-1}W^\top) + F_{\Omega}(A)  \ \ \ \forall W\in\mathbb{R}^{n \times T}
$$
Here we report the example of the nuclear norm $\|\cdot\|_*$, that has already been observed in similar form in~\cite{argyriou08,grave11} and that can be easily derived from Prop.~\ref{prop:p_solution} for the case $p=1$.
$$
\|W\|_* = \frac{1}{2} \underset{A \in S_{++}^T}{inf} trace(WA^{-1}W^\top) + trace(A).
$$
Indeed, from Prop.~\eqref{prop:p_solution} we have that the solution to the minimization problem is $A_* = \sqrt(W^top W)$ and therefore, the minimum of such functional will be exactly $trace(\sqrt{WW^\top}) = \|W\|_*$.

\section*{Impose Tasks Relationships by enforcing structure on the feature space}\label{sec:covXcovY}

Relations among tasks can be also modeled by enforcing shared structures on the input space. For instance in~\cite{argyriou08}, the authors generalized a feature selection framework to the multi-task setting by formulating the linear problem
\begin{equation}\label{eq:argy_original}
\underset{U\in O^d, M\in\mathbb{R}^{d \times T}}{\text{minimize}} \ \ V(Y,XUM) + \gamma \|M\|_{2,1}
\end{equation}
where $X\in\mathbb{R}^{n \times d}$ is the matrix whose $i$-th row corresponds to the input vector $x_i\in\mathbb{R}^d$ and the $(2,1)$-norm $\|M\|_{2,1}=\sum_{k=1}^d \|M^k\|_2$ is introduced to enforce sparsity among the rows $M^k$ of $M$. This penalty generalizes feature selection to the multi-task case by directly manipulating the covariance on the input space. However, since input and output distributions are connected by the training data, it is reasonable to expect this process to indirectly affect also the covariance on the output space. Indeed, in this Section we present an interesting result connecting multi-task problems that impose structure on the input covariance and problems that instead aim to control the output covariance (i.e. in the form of~\eqref{eq:convex_equivalence}).\\
To show this connection, we need to discuss in more detail the work in~\cite{argyriou08}. Although~\eqref{eq:argy_original} is not convex, the authors prove that there exists an equivalent convex formulation of the form
\begin{equation}\label{eq:argy_linear}
\underset{\substack{W\in \mathbb{R}^{d \times T}, D\in S_+^d,\\ \ran(D)\supseteq\ran(W), tr(D)\leq1}}{\text{minimize}} \ \ V(Y,XW) + \gamma \ tr(W^\top D^\dagger W).
\end{equation}
The authors then proceed to generalize this framework to the nonlinear case using the advantages of the RKHS notation. In this setting, the original idea of identifying a low dimensional set of directions in the feature space translates naturally to the problem of finding a small set of orthogonal directions in the Hilbert space. To this end, the authors perform a preprocessing step whose goal is to identify an orthonormal basis of functions $\psi_1,\dots\psi_\ell\in\mathcal{H}_k$ for set spanned by the $k(x_i,\cdot)$ and define a matrix $\widetilde{K}\in\mathbb{R}^{n \times \ell}$ such that $\widetilde{K}_{ij} = \psi_j(x_i)$. A possible way to do this is by considering a eigenvalue decomposition $U \Sigma U^\top$ of $K$ and taking $\widetilde{K} = U \Sigma^{1/2}$ (taking out from $\Sigma^{1/2}$ the columns equal to zero). It is easy to show that the standard learning problem in RKHS settings can be cast equivalently in this new notation. However, this framework has the further advantage that it can be generalized to take into account the eventuality of a transformation in the feature space, leading to the extension of problem~\eqref{eq:argy_linear} for the non linear case
\begin{equation}\label{eq:covX_special}
\underset{\substack{B\in \mathbb{R}^{\ell \times T}, D\in S_+^\ell,\\ \ran(D)\supseteq\ran(B), tr(D)\leq1}}{\text{minimize}} \ \ V(Y,\widetilde{K}B) + \gamma \ tr(B^\top D^\dagger B)
\end{equation}
As can be noticed, the structure of problem~\eqref{eq:covX_special} is very similar to the one of problem~\eqref{eq:convex_equivalence} and indeed, as stated in Corollary~\ref{cor:covX_special} the two are equivalent when trace regularization is imposed on~\eqref{eq:convex_equivalence}. However, as shown in Theorem~\ref{teo:covXcovY}, a more general equivalence holds.

\begin{theorem}\label{teo:covXcovY}
Let $\lambda>0$, $p\geq1$, $\mathbf{R}^{n \times T}$, $\{x_i,y_i\}_{i=1}^n\subset\mathbb{R}^d\times\mathbb{R}^T$ a set of input-output pairs with  $\mathbf{y}\in\mathbb{R}^{n \times T}$ the matrix whose $i$-th row corresponds to $y_i$. Let $\psi_1,\dots,\psi_\ell\in\mathcal{H}_k$ be an orthonormal basis for $span \{k(x_i,\cdot\}_{i=1}^n$ and $\widetilde{K}\in\mathbb{R}^{n \times \ell}$ with $\widetilde{K}_{ij}=\psi_j(x_i)$. Then
\begin{equation}\label{eq:covX}\tag{$\mathcal{T}$}
\underset{\substack{B\in \mathbb{R}^{\ell \times T}, D\in S_+^\ell,\\ \ran(D)\supseteq\ran(B)}}{\text{minimize}} \ \ S(B,D) = V(Y,\widetilde{K} B) + tr(B^*D^\dagger B) + \lambda \ \|D\|_p
\end{equation}
is a convex optimization problem equivalent to~\eqref{eq:convex_equivalence} with penalty function $F(A)=\|A\|_p$. In particular the two problems achieve the same minimum and, given a minimizer for one problem it is possible to obtain a solution for the other and vice-versa.
\end{theorem}

The crucial aspect of the proof of Theorem~\ref{teo:covXcovY} (which we prove below) consists in identifying the two mappings that allow to obtain a minimizer for problem~\eqref{eq:convex_equivalence} from a solution of~\eqref{eq:covX} and vice-versa.\\ %The structure of these mappings depends on the fact that Proposition~\eqref{eq:alignment} is valid also for problems of the form~\eqref{eq:covX}. \\
%Theorem~\eqref{teo:covXcovY} states a direct equivalence between learning a kernel on the output space or learning a covariance matrix on the input space. 
As a corollary of Theorem~\eqref{teo:covXcovY} we get the exact equivalence to the problem proposed in~\cite{argyriou08}.
\begin{corollary}\label{cor:covX_special}
Problem~\eqref{eq:covX_special} is equivalent to~\eqref{eq:covX} for $p=1$. In particular the two problems achieve the same minimum for $\lambda = \gamma^2/4$. As a consequence of Theorem~\ref{teo:covXcovY} this implies also that~\eqref{eq:covX_special} is also equivalent to~\eqref{eq:convex_equivalence} when $F(\cdot)=\|\cdot\|_1=tr(\cdot)$.
\end{corollary}
This result follows from the direct comparison of the minimizers for the  problems~\eqref{eq:covX} (from Proposition~\ref{prop:p_solution}) and~\eqref{eq:covX_special} (from~\cite{argyriou08}). Notice, that although equivalent as convex optimizations, it is in general more convenient to solve problems in the form~\eqref{eq:convex_equivalence} rather than~\eqref{eq:covX} since in most cases $T << \ell$.

\begin{proof}\textit{Theorem~\ref{teo:covXcovY}.}

From the discussion in~\cite{argyriou08} we can rewrite problem~\eqref{eq:convex_equivalence} in the equivalent formulation
\begin{equation}\label{eq:covY}\tag{$\mathcal{U}$}
\underset{\substack{B\in \mathbb{R}^{\ell \times T}, A\in S_+^T,\\ \ran(A)\supseteq\ran(B^\top)}}{\text{minimize}} \ T(B,A) = V(Y,\widetilde{K} B) + tr(A^\dagger B^\top B) + \lambda \ \|A\|_p
\end{equation}
Therefore, to prove Theorem~\ref{teo:covXcovY} it is sufficient to show that problem~\eqref{eq:covX} and~\eqref{eq:covY} are equivalent. Assume without loss of generality $T \leq \ell$. Consider an arbitrary matrix $B\in\mathbb{R}^{\ell \times T}$ and a singular value decomposition $B= V \left( \begin{array}{c}\Sigma \\ 0 \end{array} \right) U^\top$ where $0\in\mathbb{R}^{(\ell-T) \times T}$ identifies a matrix of all zeros, $V\in O^\ell, U\in O^T$ and $\Sigma\in S_+^T$ a diagonal matrix with eigenvalues in descending order. From Propositon~\ref{prop:alignment}, we obtain that the minimizers of the two functions $S(B,\cdot)$ and $T(B,\cdot)$ are unique and can be written respectively in the forms
$$
D_B=V\left(\begin{array}{cc} \Gamma_D & 0 \\ 0 & 0 \end{array}\right) V^\top \in S_+^\ell \text{ \ \ \ and \ \ \ \ } A_B = U \Gamma_A U^\top \in S_+^T
$$
where $\Gamma_D,\Gamma_A\in S_+^T$ have same sparsity pattern as $\Sigma$ and the zero matrices in the formulation of $D_B$ are of appropriate  dimension. We can therefore write the minimum value achieved by $S(B,\cdot)$ as $S(B,D_B) = V(Y,\widetilde{K}B) + tr(\Gamma_D^\dagger \Sigma^2) + \lambda \|\Gamma_D\|_p$ and the minimum achieved by $T(B,\cdot)$ as $T(B,A_B) = V(Y,\widetilde{K}B) + tr(\Gamma_A^\dagger \Sigma^2) + \lambda \|\Gamma_A\|_p$. In the light of these equations, it can be easily cheked that by setting $A_B^{(D)} = U \Gamma_D U^\top \in S_+^T$ we have 
$$
S(B,D_B) = T(B,A_B^{(D)}) \geq T(B,A_B)
$$ 
where the inequality follows from the fact that $A_B$ is a minimizer for $T(B,\cdot)$. Analogously, we can design a matrix $D_B^{(A)}\in S_+^\ell$ such that $T(B,A_B)=S(B,D_B^{(A)}) \geq S(B,D_B)$. Since the minimizers $A_B$ and $D_B$ are unique, it follows that $\Gamma_D=\Gamma_A$. In the perspective of this result, we have that for any minimizer $(B_*,D_*)\in\mathbb{R}^{\ell \times T} \times S_+^\ell$ for~\eqref{eq:covX}, the couple $(B_*,A_{B_*}^{(D_*)})\in\mathbb{R}^{\ell \times T} \times S_+^T$ is a minimizer for~\eqref{eq:covY} and furthermore, the two functions achieve the same minimum value. The same result holds in the opposite direction.
\end{proof}

\end{document}